\documentclass{article} 
\usepackage{iclr2020_conference,times}

\usepackage{amsmath,amsfonts,bm}

\def\eqref#1{equation~\ref{#1}}

\def\1{\bm{1}}

\DeclareMathAlphabet{\mathsfit}{\encodingdefault}{\sfdefault}{m}{sl}
\SetMathAlphabet{\mathsfit}{bold}{\encodingdefault}{\sfdefault}{bx}{n}

\usepackage{hyperref}
\usepackage{url}

\usepackage[utf8]{inputenc} 
\usepackage[T1]{fontenc}    
\usepackage{hyperref}       
\usepackage{url}            
\usepackage{booktabs}       
\usepackage{amsfonts}       
\usepackage{nicefrac}       
\usepackage{microtype}      

\usepackage{amsmath}
\usepackage{amssymb}
\usepackage{array,tabularx}
\usepackage[english]{babel}
\usepackage{bm}
\usepackage{comment}
\usepackage{float}
\usepackage{graphicx}
\usepackage{wrapfig}
\usepackage{multirow}
\usepackage{caption}
\usepackage{subfig}

\usepackage{algorithm}
\usepackage{algorithmic}

\usepackage{amsthm}
\newtheorem{theorem}{Theorem}
\newtheorem{proposition}{Proposition}

\newcommand{\vc}[1]{\mathbf{#1}}
\newcommand{\ci}{\perp\!\!\!\perp}
\newcommand{\cl}[1]{\mathcal{#1}}

\newcommand{\hist}[1]{\bar{\vc{#1}}}

\definecolor{mygreen}{RGB}{0, 98, 0}

\title{Estimating counterfactual treatment \\ outcomes over time through adversarially balanced representations}

\author{Ioana Bica \\
Department of Engineering Science\\
University of Oxford, Oxford, UK\\
The Alan Turing Institute, London, UK \\
\texttt{ioana.bica@eng.ox.ac.uk} \\
\And
Ahmed M. Alaa \\
Department of Electrical Engineering \\
University of California, Los Angeles, USA \\
\texttt{ahmedmalaa@ucla.edu} \\
\AND
James Jordon \\
Department of Engineering Science\\
University of Oxford, Oxford, UK\\
\texttt{james.jordon@wolfson.ox.ac.uk} \\
\And
Mihaela van der Schaar \\
University of Cambridge, Cambridge, UK\\
University of California, Los Angeles, USA \\
The Alan Turing Institute, London, UK \\
\texttt{mv472@cam.ac.uk} \\
}

\iclrfinalcopy 
\begin{document}

\maketitle

\begin{abstract}
Identifying when to give treatments to patients and how to select among multiple treatments over time are important medical problems with a few existing solutions. In this paper, we introduce the Counterfactual Recurrent Network (CRN), a novel sequence-to-sequence model that leverages the increasingly available patient observational data to estimate treatment effects over time and answer such medical questions. To handle the bias from time-varying confounders, covariates affecting the treatment assignment policy in the observational data, CRN uses domain adversarial training to build balancing representations of the patient history. At each timestep, CRN constructs a treatment invariant representation which removes the association between patient history and treatment assignments and thus can be reliably used for making counterfactual predictions. On a simulated model of tumour growth, with varying degree of time-dependent confounding, we show how our model achieves lower error in estimating counterfactuals and in choosing the correct treatment and timing of treatment than current state-of-the-art methods.
\end{abstract}

	\section{Introduction}	
	As clinical decision-makers are often faced with the problem of choosing between treatment alternatives for patients, reliably estimating their effects is paramount. While clinical trials represent the gold standard for causal inference, they are expensive, have a few patients and narrow inclusion criteria \citep{booth2014randomised}. Leveraging the increasingly available observational data about patients, such as electronic health records, represents a more viable alternative for estimating treatment effects. 
	
	A large number of methods have been proposed for performing causal inference using observational data in the static setting \citep{johansson2016learning, shalit2017estimating, alaa2017bayesian, li2017matching, yoon2018ganite, alaa2018limits, yao2018representation} and only a few methods address the longitudinal setting \citep{xu2016bayesian, roy2016bayesian, soleimani2017treatment, schulam2017reliable, lim2018forecasting}. However, estimating the effects of treatments over time poses unique opportunities such as understanding how diseases evolve under different treatment plans, how individual patients respond to medication over time, but also which are optimal timings for assigning treatments, thus providing new tools to improve clinical decision support systems. 
		
	\begin{figure}[t]
	\centering
	\includegraphics[width=0.97\columnwidth]{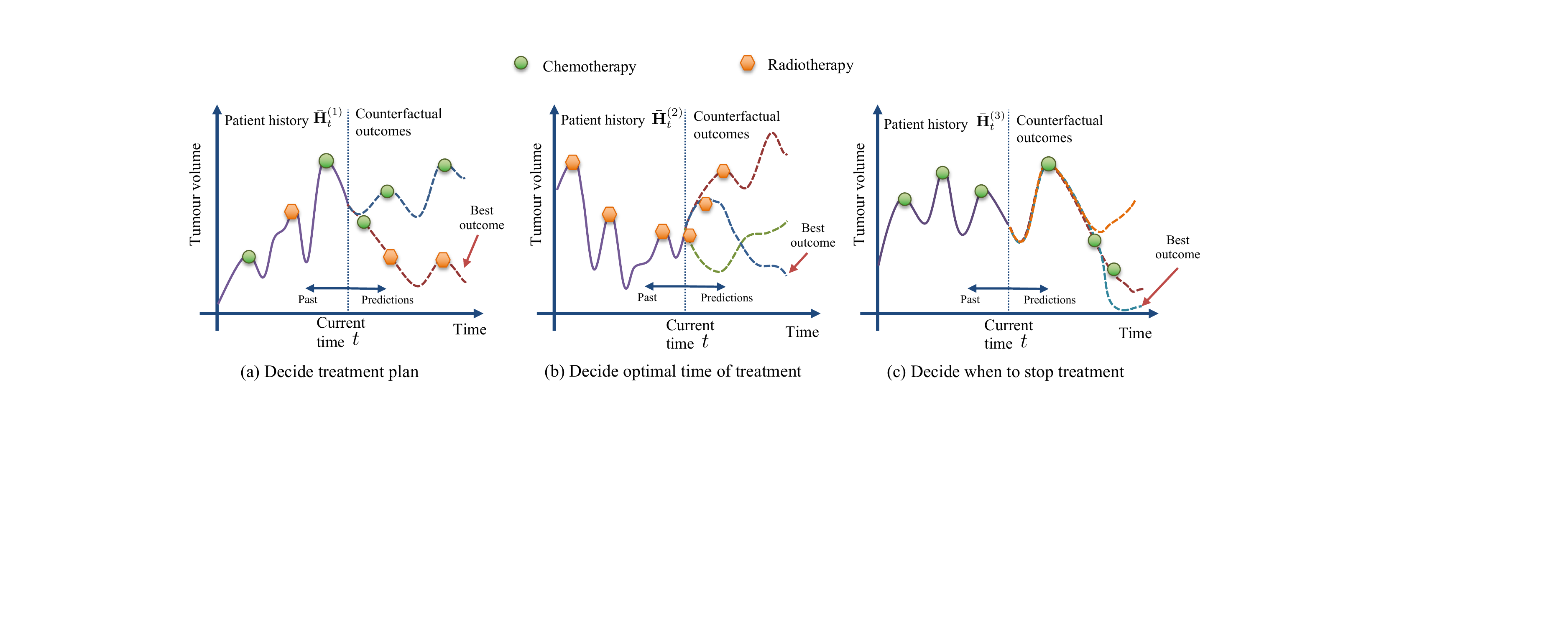}
	\caption{Applicability of CRN in cancer treatment planning. We illustrate 3 patients with different covariate and treatment histories $\hist{H}_t$. For a current time $t$, CRN can predict counterfactual trajectories (the coloured dashed branches) for planned treatments in the future. Through the counterfactual predictions, we can decide which treatment plan results in the best patient outcome (in this case, the lowest tumour volume). This way, CRN can be used to perform all of the following: choose optimal treatments (a), find timing when treatment is most effective (b) decide when to stop treatment (c).}
	\vspace{-0.7cm}
	\label{fig:treatment_scenarios}
\end{figure}

The biggest challenge when estimating the effects of time-dependent treatments from observational data involves correctly handling the time-dependent confounders: patient covariates that are affected by past treatments which then influence future treatments and outcomes \citep{platt2009time}. For instance, consider that treatment A is given when a certain patient covariate (e.g. white blood cell count) has been outside of normal range values for several consecutive timesteps. Suppose also that this patient covariate was itself affected by the past administration of treatment B. If these patients are more likely to die, without adjusting for the time-dependent confounding (e.g. the changes in the white blood cell count over time), we will incorrectly conclude that treatment A is harmful to patients. Moreover, estimating the effect of a different sequence of treatments on the patient outcome would require not only adjusting for the bias at the current step (in treatment A), but also for the bias introduced by the previous application of treatment B. 

Existing methods for causal inference in the static setting cannot be applied in this longitudinal setting since they are designed to handle the cross-sectional set-up, where the treatment and outcome depend only on a static value of the patient covariates. If we consider again the above example, these methods would not be able to model how the changes in patient covariates over time affect the assignment of treatments and they would also not be able to estimate the effect of a sequence of treatments on the patient outcome (e.g. sequential application of treatment A followed by treatment B). Different models that can handle these temporal dependencies in the observational data and varying-length patient histories are needed for estimating treatment effects over time.

Time-dependent confounders are present in observational data because doctors follow policies: the history of the patients' covariates and the patients' response to past treatments are used to decide future treatments \citep{mansournia2012effect}. The direct use of supervised learning methods will be biased by the treatment policies present in the observational data and will not be able to correctly estimate counterfactuals for different treatment assignment policies.

Standard methods for adjusting for time-varying confounding and estimating the effects of time-varying exposures are based on ideas from epidemiology. The most widely used among these are Marginal Structural Models (MSMs) \citep{robins2000marginal, mansournia2012effect} which use the inverse probability of treatment weighting (IPTW) to adjust for the time-dependent confounding bias. Through IPTW, MSMs create a pseudo-population where the probability of treatment does not depend on the time-varying confounders. However, MSMs are not robust to model misspecification in computing the IPTWs. MSMs can also give high-variance estimates due to extreme weights; computing the IPTW involves dividing by probability of assigning a treatment conditional on patient history which can be numerically unstable if the probability is small.

	We introduce the Counterfactual Recurrent Network (CRN), a novel sequence-to-sequence architecture for estimating treatment effects over time. CRN leverages recent advances in representation learning \citep{bengio2012representation} and domain adversarial training \citep{ganin2016domain} to overcome the problems of existing methods for causal inference over time. Our main contributions are as follows.  
	
    \textbf{Treatment invariant representations over time.} CRN constructs treatment invariant representations at each timestep in order to break the association between patient history and treatment assignment and thus removes the bias from time-dependent confounders. For this, CRN uses domain adversarial training \citep{ganin2016domain, li2018deep, sebag2019multi} to trade-off between building this balancing representation and predicting patient outcomes. We show that these representations remove the bias from time-varying confounders and can be reliably used for estimating counterfactual outcomes. This represents the first work that introduces ideas from domain adaptation to the area of estimating treatment effects over time. In addition, by building balancing representations, we propose a novel way of removing the bias introduced by time-varying confounders.  
    
	\textbf{Counterfactual estimation of future outcomes.} To estimate counterfactual outcomes for treatment plans (and not just single treatments), we integrate the domain adversarial training procedure as part of a sequence-to-sequence architecture. CRN consists of an encoder network which builds treatment invariant representations of the patient history that are used to initialize the decoder. The decoder network estimates outcomes under an intended sequence of future treatments, while also updating the balanced representation. By performing counterfactual estimation of future treatment outcomes, CRN can be used to answer critical medical questions such as deciding when to give treatments to patients, when to start and stop treatment regimes, and also how to select from multiple treatments over time. We illustrate in Figure \ref{fig:treatment_scenarios} the applicability of our method in choosing optimal cancer treatments. 
	
	In our experiments, we evaluate CRN in a realistic set-up using a model of tumour growth \citep{geng2017prediction}. We show that CRN achieves better performance in predicting counterfactual outcomes, but also in choosing the right treatment and timing of treatment than current state-of-the-art methods. 

	\section{Related work}
	
    We focus on methods for estimating  treatment effects over time and for building balancing representations for causal inference. A more in-depth review of related work is in Appendix \ref{apx:related_work}. 

	\textbf{Treatment effects over time.} Standard methods for estimating the effects of time-varying exposures were first developed in the epidemiology literature and include the g-computation formula, Structural Nested Models and Marginal Structural Models (MSMs) \citep{robins1986new, robins1994correcting, robins2000marginal, robins2008estimation}. Originally, these methods have used predictors performing logistic/linear regression which makes them unsuitable for handling complex time-dependencies \citep{hernan2001marginal, mansournia2012effect, mortimer2005application}. To address these limitations, methods that use Bayesian non-parametrics or recurrent neural networks as part of these frameworks have been proposed. \citep{xu2016bayesian, roy2016bayesian, lim2018forecasting}. 
	
	To begin with, \cite{xu2016bayesian} use Gaussian processes to model discrete patient outcomes as a generalized mixed-effects model and uses the $g$-computation method to handle time-varying confounders. \citet{soleimani2017treatment} extend the approach in \citet{xu2016bayesian} to the continuous time-setting and model treatment responses using linear time-invariant dynamical systems. \cite{roy2016bayesian} use Dirichlet and Gaussian processes to model the observational data and estimate the IPTW in Marginal Structural Models. \cite{schulam2017reliable} build upon work from \cite{lok2008statistical, arjas2004causal} and use marked point processes and Gaussian processes to learn causal effects in continuous-time data. These Bayesian non-parametric methods make strong assumptions about model structure and consequently cannot handle well heterogeneous treatment effects arising from baseline variables \citep{soleimani2017treatment, schulam2017reliable} and multiple treatment outcomes \citep{xu2016bayesian, schulam2017reliable}. 
	
	 The work most related to ours is the one of \citet{lim2018forecasting} which improves on the standard MSMs by using recurrent neural networks to estimate the inverse probability of treatment weights (IPTWs). \cite{lim2018forecasting} introduces Recurrent Marginal Structural Networks (RMSNs) which also use a sequence-to-sequence deep learning architecture to forecast treatment responses in a similar fashion to our model. However, RMSNs require training additional RNNs to estimate the propensity weights and does not overcome the fundamental problems with IPTWs, such as the high-variance of the weights. Conversely, CRN takes advantage of the recent advances in machine learning, in particular, representation learning to propose a novel way of handling time-varying confounders. 
			
	\textbf{Balancing representations for treatment effect estimation.} Balancing the distribution of control and treated groups has been used for counterfactual estimation in the static setting. The methods proposed in the static setting for balancing representations are based on using discrepancy measures in the representation space between treated and untreated patients, which do not generalize to multiple treatments \citep{johansson2016learning, shalit2017estimating, li2017matching, yao2018representation}. Moreover, due to the sequential assignment of treatments in the longitudinal setting, and due to the change of patient covariates over time according to previous treatments, the methods for the static setting are not directly applicable to the time-varying setting \citep{hernan2000marginal, mansournia2012effect}.

	\section{Problem formulation}
	
	Consider an observational dataset $\cl{D} = \big\{ \{\vc{x}^{(i)}_{t},  \vc{a}^{(i)}_{t}, \vc{y}_{t+1}^{(i)} \}_{t=1}^{T^{(i)}} \cup \{\vc{v}^{(i)} \}  \big\}_{i=1}^N$ consisting of information about $N$ independent patients. For each patient $(i)$, we observe time-dependent covariates $\vc{X}_t^{(i)} \in \cl{X}_t$, treatment received $\vc{A}^{(i)}_t \in \{A_1, \dots A_K\} = \cl{A}$ and outcomes $\vc{Y}^{(i)}_{t+1} \in \mathcal{Y}_{t+1}$ for $T^{(i)}$ discrete timesteps. The patient can also have baseline covariates $\vc{V}^{(i)} \in \cl{V}$ such as gender and genetic information. Note that the outcome $\vc{Y}^{(i)}_{t+1}$ will be part of the observed covariates $\vc{X}^{(i)}_{t+1}$. For simplicity, the patient superscript $(i)$ will be omitted unless explicitly needed.

	We adopt the potential outcomes framework proposed by \citep{neyman1923applications, rubin1978bayesian} and extended by \citep{robins2008estimation} to account for time-varying treatments. Let $\mathbf{Y}[\hist{a}]$ be the potential outcomes, either factual or counterfactual, for each possible course of treatment $\hist{a}$. Let $\bar{\vc{H}}_{t} = (\bar{\vc{X}}_{t}, \bar{\vc{A}}_{t-1}, \vc{V}) $ represent the history of the patient covariates $\bar{\vc{X}}_{t} = (\vc{X}_{1}, \dots, \vc{X}_{t})$,  treatment assignments $\hist{{A}}_{t} = (\vc{A}_{1}, \dots, \vc{A}_t)$ and static features $\vc{V}$. We want to estimate:
	\begin{equation}
	\mathbb{E}(\vc{Y}_{t+\tau}[\hist{a}(t, t+\tau -1 )] | \bar{\vc{H}}_{t}),
	\end{equation} 
	where  $\hist{a}(t, t+\tau-1) = [\vc{a}_{t}, \dots \vc{a}_{{t+\tau -1 }}]$ represents a possible sequence of treatments from timestep $t$ just until before the potential outcome $\vc{Y}_{t+\tau}$ is observed. We make the standard assumptions \citep{robins2000marginal, lim2018forecasting} needed to identify the treatment effects: consistency, positivity and no hidden confounders (sequential strong ignorability). See Appendix \ref{apx:assumptions} for more more details.

		\section{Counterfactual Recurrent Network} \label{sec:crn}
		
		The observational data can be used to train a supervised learning model to forecast: $\mathbb{E}(\vc{Y}_{t+\tau}\mid \hist{A}(t, t+\tau -1 ) =  \hist{a}(t, t+\tau -1 ), \bar{\vc{H}}_{t})$. However, without adjusting for the bias introduced by time-varying confounders, this model cannot be reliably used for making causal predictions \citep{robins2000marginal, robins2008estimation, schulam2017reliable}. The Counterfactual Recurrent Network (CRN) removes this bias through domain adversarial training and estimates the counterfactual outcomes $\mathbb{E}(\vc{Y}_{t+\tau}[\hist{a}(t, t+\tau -1 )] | \bar{\vc{H}}_{t})$,
		for any intended future treatment assignment $\hist{a}(t, t+\tau -1 )$.

		\textbf{Balancing representations.}  The history $\bar{\vc{H}}_{t} = (\bar{\vc{X}}_{t}, \bar{\vc{A}}_{t-1}, \vc{V}) $ of the patient contains the time-varying confounders $\hist{X}_t$ which bias the treatment assignment $\vc{A}_t \in \{ A_{1}, \dots A_{K} \}$  in the observational dataset. Inverse probability of treatment weighting, as performed by MSMs, creates a pseudo-population where the probability of treatment $\vc{A}_t$ does not depend on the time-varying confounders \citep{robins2000marginal}.  In this paper, we propose instead building a representation of the history $\bar{\vc{H}}_{t}$ that is not predictive of the treatment $\vc{A}_t$. This way, we remove the association between history, containing the time-varying confounders $\hist{X}_t$, and current treatment $\vc{A}_t$. \cite{robins1999association} shows that in this case, the estimation of counterfactual treatment outcomes is unbiased. See Appendix \ref{apx:time_dependent_confounding} for details and for an example of a causal graph with time-dependent confounders.

		Let $\Phi$ be the representation function that maps the patient history  $\bar{\vc{H}}_{t}$ to a representation space $ \cl{R}$. To obtain unbiased treatment effects, $\Phi$ needs to construct treatment invariant representations such that $P(\Phi(\hist{H}_t)\mid \vc{A}_t =  A_{1}) = \dots = P( \Phi(\hist{H}_t)\mid \vc{A}_t = A_{K})$. To achieve this and to estimate counterfactual outcomes under a planned sequence of treatments, we integrate the domain adversarial training framework proposed by \citet{ganin2016domain} and extended by \citet{sebag2019multi} to the multi-domain learning setting, into a sequence-to-sequence architecture. In our case, the different treatments at each timestep are considered the different domains. Note that the novelty here comes from the use of domain adversarial training to handle the bias from the time-dependent confounders, rather than the use of sequence-to-sequence models, which have already been applied to forecast treatment responses \citep{lim2018forecasting}. Figure \ref{fig:sequence_prediction} illustrates our model architecture.
		
		 \begin{figure}[t]
			\centering
			\vspace{-0.1cm}
			\includegraphics[width=\columnwidth]{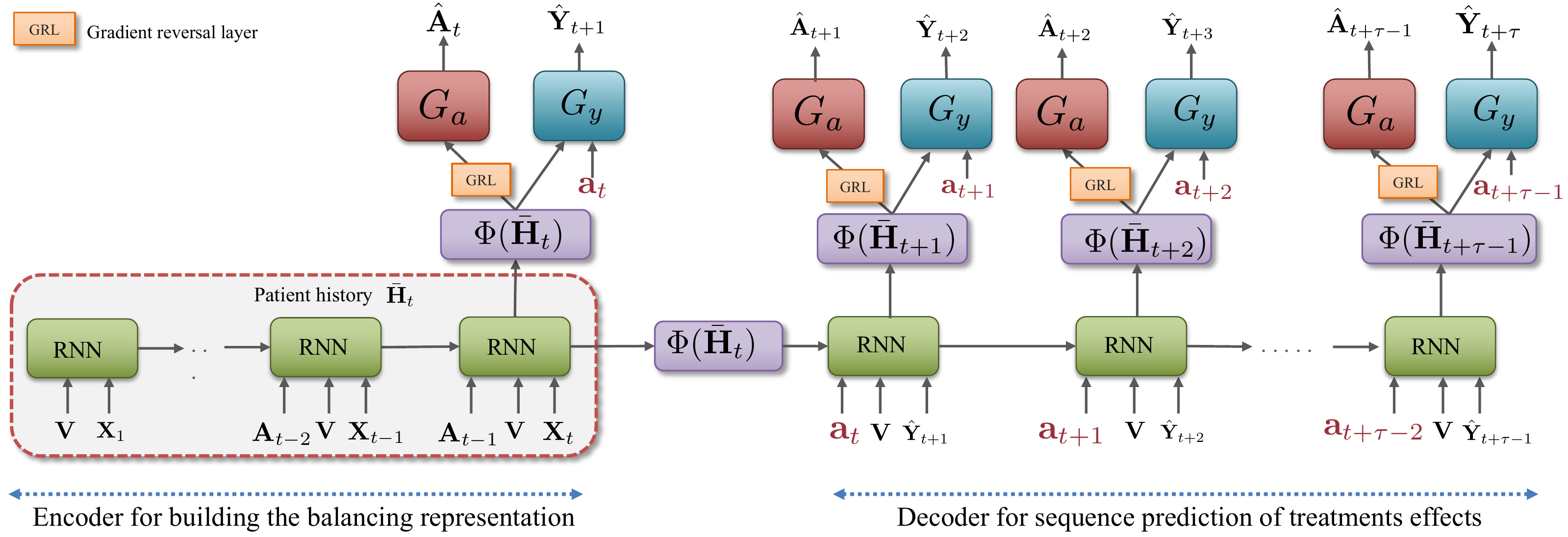}
			\caption{CRN architecture. Encoder builds representation $\Phi(\hist{H}_t)$ that maximizes loss of treatment classifier $G_a$ and minimizes loss of outcome predictor $G_y$. $\Phi(\hist{H}_t)$ is used to initialize the decoder, which continues to update it to predict counterfactual outcomes of a sequence of future treatments.}
			\label{fig:sequence_prediction}
			\vspace{-0.5cm}
		\end{figure}

		\textbf{Encoder.} The encoder network uses an RNN, with LSTM unit \citep{hochreiter1997long}, to process the history of treatments $\hist{A}_{t-1}$, covariates $\hist{X}_t$ and baseline features $\vc{V}$ to build a treatment invariant representation $\Phi(\hist{H}_t)$, but also to predict one-step-ahead outcomes $\vc{Y}_{t+1}$. To achieve this, the encoder network aims to maximize the loss of the treatment classifier $G_a$ and minimize the loss of the outcome predictor network $G_y$. This way, the balanced representation $\Phi(\hist{H}_t)$ is not predictive of the assigned treatment $\vc{A}_t$, but is discriminative enough to estimate the outcome $\vc{Y}_{t+1}$. To train this model using gradient descent, we use the Gradient Reversal Layer \citep{ganin2016domain}. 
		
		\textbf{Decoder.} The decoder network uses the balanced representation computed by the encoder to initialize the state of an RNN that predicts the counterfactual outcomes for a sequence of future treatments. During training, the decoder uses as input the outcomes from the observational data $(\vc{Y}_{t+1}, \dots \vc{Y}_{t+\tau-1})$, the static patient features $\vc{V}$ and the intended sequence of treatments $\hist{a}(t, t+\tau - 1)$. The decoder is trained in a similar way to the encoder to update the balanced representation and to estimate the outcomes. During testing, we do not have access to ground-truth outcomes; thus, the outcomes predicted by the decoder $(\hat{\vc{Y}}_{t+1}, \dots \hat{\vc{Y}}_{t+\tau-1})$ are auto-regressively used instead as inputs. By running the decoder with different treatment settings, and by auto-regressively feeding back the outcomes, we can determine when to start and end different treatments, which is the optimal time to give the treatment and which treatments to give over time to obtain the best patient outcomes. 

		The representation $\Phi(\hist{H}_t)$ is built by applying a fully connected layer, with Exponential Linear Unit (ELU)
		activation to the output of the LSTM. The treatment classifier $G_a$ and the predictor network $G_y$ consist of a hidden layer each, also with ELU activation. The output layer of $G_a$ uses softmax activation, while the output layer of $G_y$ uses linear activation for continuous predictions. For categorical outcomes, softmax activation can be used. We follow an approach similar to \citet{lim2018forecasting} and we split the encoder and decoder training into separate steps. See Appendix \ref{apx:training_procedure} for details.  	  
		
		The encoder and decoder networks use variational dropout \citep{gal2016theoretically} such that the CRN can also give uncertainty intervals for the treatment outcomes. This is particularity important in the estimation of treatment effects, since the model predictions should only be used when they have high confidence.  Our model can also be modified to allow for irregular samplings of observations by using a PhasedLSTM \citep{neil2016phased}.

	\section{Adversarially balanced representation over time} \label{sec:adv_bal_rep}

	At each timestep $t$, let the $K$ different possible treatments $\vc{A}_t \in \{ A_{1}, \dots A_{K} \}$ represent our domains. As described in Section \ref{sec:crn}, to remove the bias from time-dependent confounders, we build a representation of history $\hist{H}_t$ that is invariant across treatments: $P( \Phi(\hist{H}_t)\mid A_{1}) = \dots = P( \Phi(\hist{H}_t)\mid A_{K})$.

	This requirement can be enforced by minimizing the distance in the distribution of $\Phi(\hist{H}_t)$ between any two pairs of treatments. \citet{kifer2004detecting, ben2007analysis}, propose measuring the disparity between distributions based on their separability by a  discriminatively-trained classifier. Let the symmetric hypothesis class $\cl{H}$ consist of the set of symmetric multiclass classifiers, such as neural network architectures. 
	The $\cl{H}$-divergence between all pairs of two distributions is defined in terms of the capacity of the hypothesis class $\cl{H}$ to discriminate between examples from the multiple distributions. Empirically, minimizing the $\cl{H}-$divergence involves building a representation where examples from the multiple domains are as indistinguishable as possible \citep{ben2007analysis, li2018deep, sebag2019multi}. \cite{ganin2016domain} use this idea to propose an adversarial framework for domain adaptation involving building a representation which achieves maximum error on a domain classifier and minimum error on an outcome predictor. Similarly, in our case, we use domain adversarial training  to build a representation of the patient history $\Phi(\hist{H}_t)$ that is both invariant to the treatment given at timestep $t$, $\vc{A}_t$ and that achieves low error in estimating the outcome $\vc{Y}_{t+1}$.

    \begin{wrapfigure}{r}{0.4\columnwidth}
		\vspace{-0.4cm}
		\begin{center}
			\includegraphics[width=0.4\columnwidth]{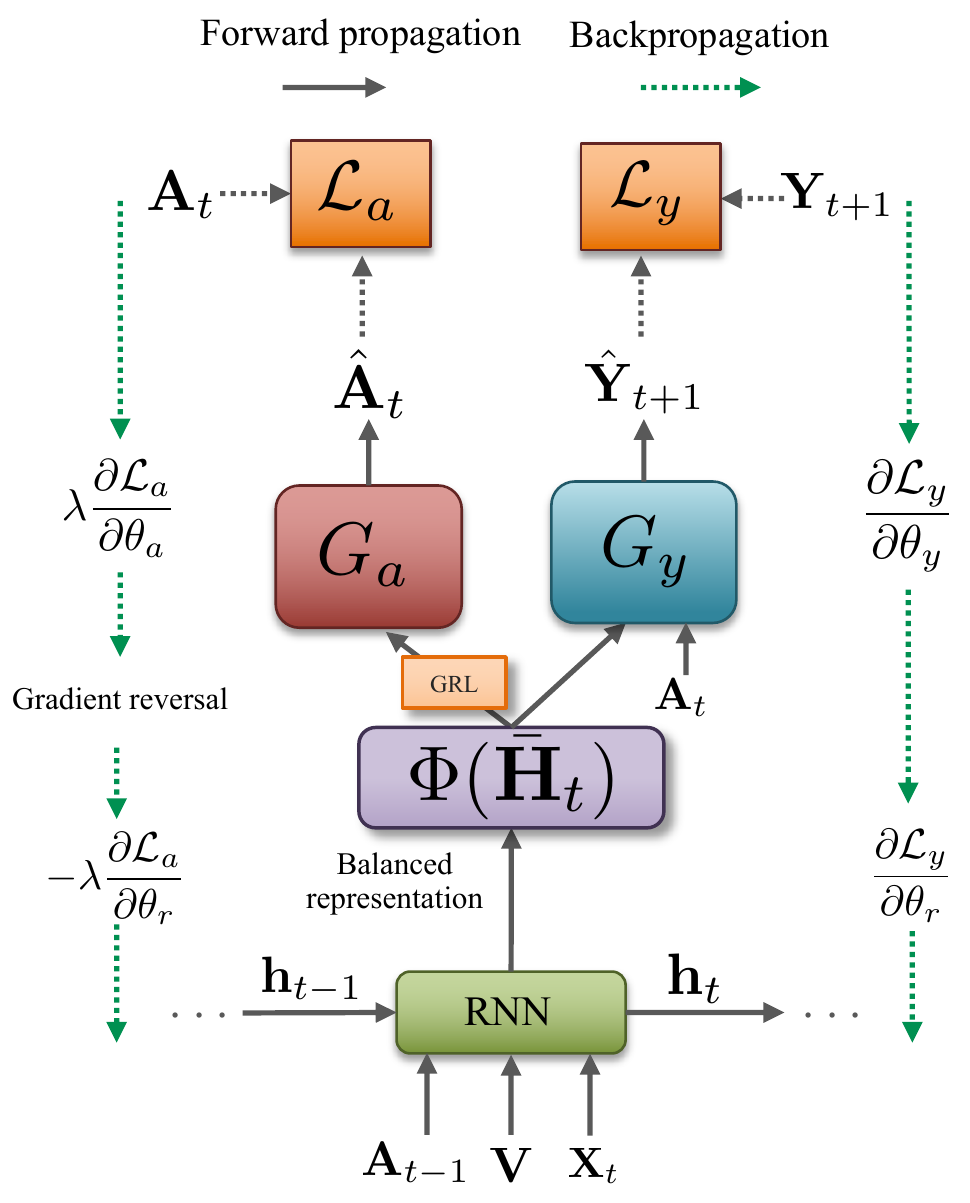}
		\end{center}
		\caption{Training procedure for building balancing representation.} 
		\vspace{-2.5cm} 
		\label{fig:adverasial_balanced_rep}
	\end{wrapfigure}	
 
	 Let $G_a(\Phi(\hist{H}_t); \theta_a)$ be the treatment classifier with parameters $\theta_a$ and let $G^j_a(\Phi(\hist{H}_t); \theta_a)$ be the output corresponding to treatment $A_j$. Let $G_y(\Phi(\hist{H}_t); \theta_y)$ be the predictor network with parameters $\theta_y$. The representation function $\Phi$ is parameterized by the parameters $\theta_r$ in the RNN: $\Phi(\hist{H}_t; \theta_r)$. Figure \ref{fig:adverasial_balanced_rep} shows the adversarial training procedure used. 
	 
	 For timestep $t$ and patient $(i)$, let $ \mathcal{L}^{(i)}_{t, a} (\theta_r, \theta_a)$ be the treatment (domain) loss and let $ \mathcal{L}^{(i)}_{t, y} (\theta_r, \theta_y)$ the outcome loss, defined as follows:
	 \begin{align}
	 \mathcal{L}^{(i)}_{t, a} (\theta_r, \theta_a) &=  - \sum_{j=1}^{K} \mathbb{I}_{\{\vc{a}^{(i)}_{t} = a_j\}} \log(G^j_a(\Phi(\hist{H}_t; \theta_r); \theta_a)) \label{eq:domain_loss} \\
	 \mathcal{L}^{(i)}_{t, y} (\theta_r, \theta_y) &=   \lVert \vc{Y}^{(i)}_{t+1} -  (G_y(\Phi(\hist{H}_t; \theta_r), \theta_y)) \rVert^2.
	 \end{align}
	 
	 If the outcome is binary, the cross-entropy loss can be used instead for $ \mathcal{L}_{t, y} $. To build treatment invariant representations and to also estimate patient outcomes, we aim to maximize treatment loss and minimize outcome loss. 
	 
	 Thus, the overall loss $\mathcal{L}^{(i)}_{t, y}$ at timestep $t$ is given by: 
		 \begin{equation}
		 \mathcal{L}^{(i)}_t(\theta_r, \theta_y, \theta_a) = \sum_{i=1}^{N} \mathcal{L}_{t, y}^{(i)} (\theta_r, \theta_y) - \lambda \mathcal{L}_{t, a}^{(i)} (\theta_r, \theta_a), 
		 \end{equation}
	where the hyperparameter $\lambda$ controls this trade-off between domain discrimination and outcome prediction.  We use the standard procedure for training domain adversarial networks from \cite{ganin2016domain} and we start off with an initial value for $\lambda$ and use an exponentially increasing schedule during training. To train the model using backpropagation, we use the Gradient Reversal Layer (GRL) \citep{ganin2016domain}. For more details about the training procedure, see Appendix \ref{apx:training_procedure}.

	By using the objective $\mathcal{L}^{(i)}_t(\theta_r, \theta_y, \theta_a)$, we reach the saddle point $(\hat{\theta}_{r},  \hat{\theta}_y, \hat{\theta}_a) $ that achieves the equilibrium between domain discrimination and outcome estimation. 
	 \begin{align}
	 (\hat{\theta}_r, \hat{\theta}_y) = \arg \min_{\theta_r, \theta_y} \mathcal{L}^{(i)}_t(\theta_r, \theta_y, \hat{\theta}_a) &&  \hat{\theta}_a = \arg \max_{\theta_a} \mathcal{L}^{(i)}_t(\hat{\theta}_r, \hat{\theta}_y, \theta_a).
	 \end{align}

	 The result stated in Theorem 1 proves that the treatment (domain) loss part of our objective (from equation \ref{eq:domain_loss}) aims to  remove the time-dependent confounding bias. 
	 
	 \begin{theorem}
	 Let $t \in \{1, 2, \dots\}$. For each $j = 1, ..., K$, let $P_j$ denote the distribution of $\bar{\vc{H}}_t$ conditional on $\vc{A}_t = A_j$ and let $P^{\Phi}_j$ denote the distribution of $\Phi(\bar{\vc{H}}_t)$ conditional on $\vc{A}_t = A_j$. Let $G_a^j$ denote the output of $G_a$ corresponding to treatment $A_j$. Then the minimax game defined by
	 \begin{align} \label{eq:minimax}
	     \min_{\Phi} \max_{G_a} \sum_{j = 1}^K &\mathbb{E}_{\bar{\vc{H}}_t \sim P_j} \Big[\log(G_a^j(\Phi(\bar{\vc{H}}_t); \theta_a))\Big] &
	     \text{subject to } \sum_{j = 1}^K &G_a^j(\Phi(\bar{\vc{H}}_t)) = 1
	 \end{align}
	 has a global minimum which is attained if and only if $P^{\Phi}_1 = P^{\Phi}_2 = ... = P^{\Phi}_K$, i.e. when the learned representations are invariant across all treatments. 
	 \end{theorem}
	 \begin{proof}
	 This result is a restatement of the one in \citet{li2018deep}. For details, see the Appendix \ref{apx:proof_theorem_1}.
	 \end{proof}

	 A good representation allows us to obtain a low error in estimating counterfactuals for all treatments, while at the same time to minimize the $\mathcal{H}$-divergence between induced marginal distributions of all the domains. We use an algorithm that directly minimizes a combination of the $\mathcal{H}-$divergence and the empirical training margin.

\section{Experiments}
	
	In real datasets, counterfactual outcomes and the degree of time-dependent confounding are not known \citep{schulam2017reliable, lim2018forecasting}. To validate the CRN\footnote{The implementation of the model can be found at \url{https://bitbucket.org/mvdschaar/mlforhealthlabpub/src/master/alg/counterfactual_recurrent_network/} and at \url{https://github.com/ioanabica/Counterfactual-Recurrent-Network}.}, we evaluate it on a  Pharmacokinetic-Pharmacodynamic model of tumour growth \citep{geng2017prediction}, which uses a state-of-the-art bio-mathematical model to simulate the combined effects of chemotherapy and radiotherapy in lung cancer patients.  The same model was used by \citet{lim2018forecasting} to evaluate RMSNs.

 	\textbf{Model of tumour growth} The volume of tumour $t$ days after diagnosis is modelled as follows:
	\begin{equation}
	\begin{aligned}
	V(t+1) =  \Big( 1 + \underbrace{\rho \text{log}\big(\dfrac{K}{V(t)}\big)}_{\text{Tumor growth}} - \underbrace{\beta_c C(t)}_{\text{Chemotherapy}} - \underbrace{\left(\alpha_r d(t) + \beta_r d(t)^2\right)}_{\text{Radiotherapy}} + \underbrace{e_t}_{\text{Noise}}  \Big)V(t) 
	\end{aligned}
	\end{equation}
	where $K, \rho, \beta_c, \alpha_r, \beta_r, e_t$ are sampled as described in \cite{geng2017prediction}. To incorporate heterogeneity in patient responses, the prior means for $\beta_c$ and  $\alpha_r$ are adjusted to create patient subgroups, which are used as baseline features. The chemotherapy  concentration $C(t)$ and radiotherapy dose $d(t)$ are modelled as described in Appendix \ref{apx:model_of_tumour_growth}. Time-varying confounding is introduced by modelling  chemotherapy and radiotherapy assignment as Bernoulli random variables, with probabilities $p_c$ and $p_r$ depending on the tumour diameter: $p_c(t) = \sigma \big(\frac{\gamma_c}{D_{\max}} (\bar{D}(t) - \delta_c )  \big)$ and $p_r(t) = \sigma \big(\frac{\gamma_r}{D_{\max}} (\bar{D}(t) - \delta_r ) \big)$
	where $\bar{D}(t)$ is the average diameter over the last $15$ days, $D_{\max} = 13 \text{cm}$, $\sigma(\cdot)$ is the sigmoid and $\delta_c = \delta_r = D_{\max}/2$.   The amount of time-dependent confounding is controlled through $\gamma_c, \gamma_r$; the higher $\gamma_{\star}$ is, the more important the history is in assigning treatments. At each timestep, there are four treatment options: no treatment,  chemotherapy,  radiotherapy, combined chemotherapy and radiotherapy. For details about data simulation, see Appendix \ref{apx:model_of_tumour_growth}.

    \textbf{Benchmarks} We used the following benchmarks for performance comparison: Marginal Structural Models (MSMs) \citep{robins2000marginal}, which use logistic regression for estimating the IPTWs and linear regression for prediction (see Appendix \ref{apx:marginal_structural_models} for  details). We also compare against the Recurrent Marginal Structural Networks (RMSNs) \cite{lim2018forecasting}, which is the current state-of-the-art model in estimating treatment responses. RMSNs use RNNs to estimate the IPTWs and the patient outcomes (details in Appendix \ref{apx:recurrent_marginal_structural_networks}). To show that standard supervised learning models do not handle the time-varying confounders we compare against an RNN and a linear regression model, which receive as input treatments and covariates to predict the outcome (see Appendix \ref{apx:baseline_RNN} for details). Our model architecture follows the description in Sections \ref{sec:crn} and \ref{sec:adv_bal_rep}, with full training details and hyperparameter optimization in Appendix \ref{apx:hyperparameter_optimization}. To show the importance of adversarial training, we also benchmark against CRN ($\lambda = 0$)  a model with the same architecture, but with $\lambda = 0$, i.e our model architecture without adversarial training.

	\subsection{Evaluate models on counterfactual predictions}
	
	Previous methods focused on evaluating the error only for factual outcomes (observed patient outcomes) \citep{lim2018forecasting}. However, to build decision support systems, we need to evaluate how well the models estimate the counterfactual outcomes, i.e patient outcomes under alternative treatment options. 	The parameters $\gamma_c$ and $\gamma_r$ control the treatment assignment policy, i.e. the degree of time-dependent confounding present in the data. We evaluate the benchmarks under different degrees of time-dependent confounding by setting $\gamma = \gamma_c = \gamma_r$. For each $\gamma$ we simulate a 10000 patients for training, 1000 for validation (hyperparameter tuning) and 1000 for out-of-sample testing. For the patients in the test set, for each time $t$, we also simulate counterfactuals $\vc{Y}_{t+1}$, represented by tumour volume $V(t+1)$, under all possible treatment options. 

     \begin{figure}[t]
        	\vspace{-0.8cm}
        	\centering
        	\subfloat[\textbf{One}-step ahead prediction]{{\includegraphics[height=4.7cm]{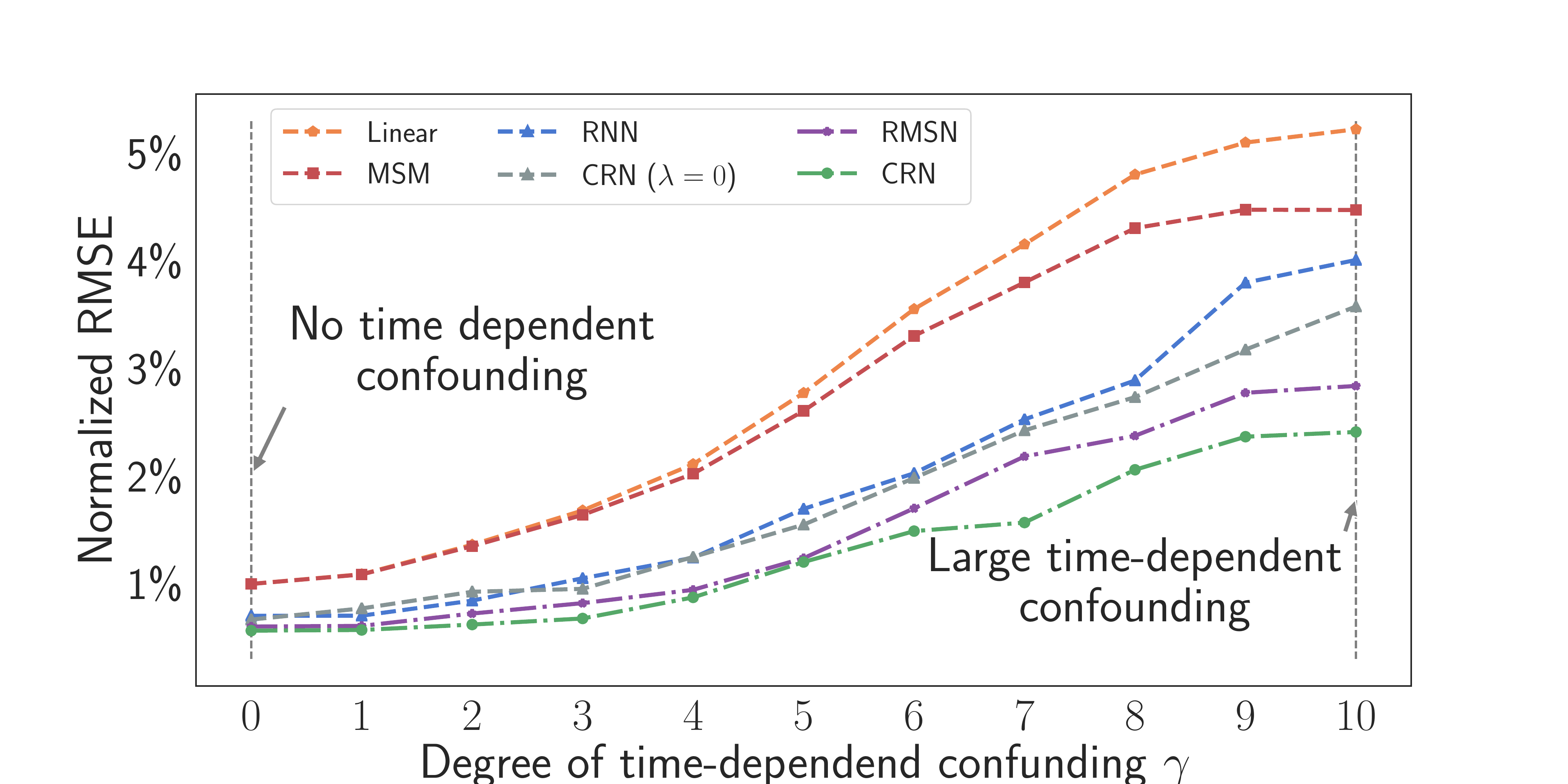} }} 
        	\quad
        	\subfloat[\textbf{Five}-step ahead prediction]{{\includegraphics[height=4.7cm]{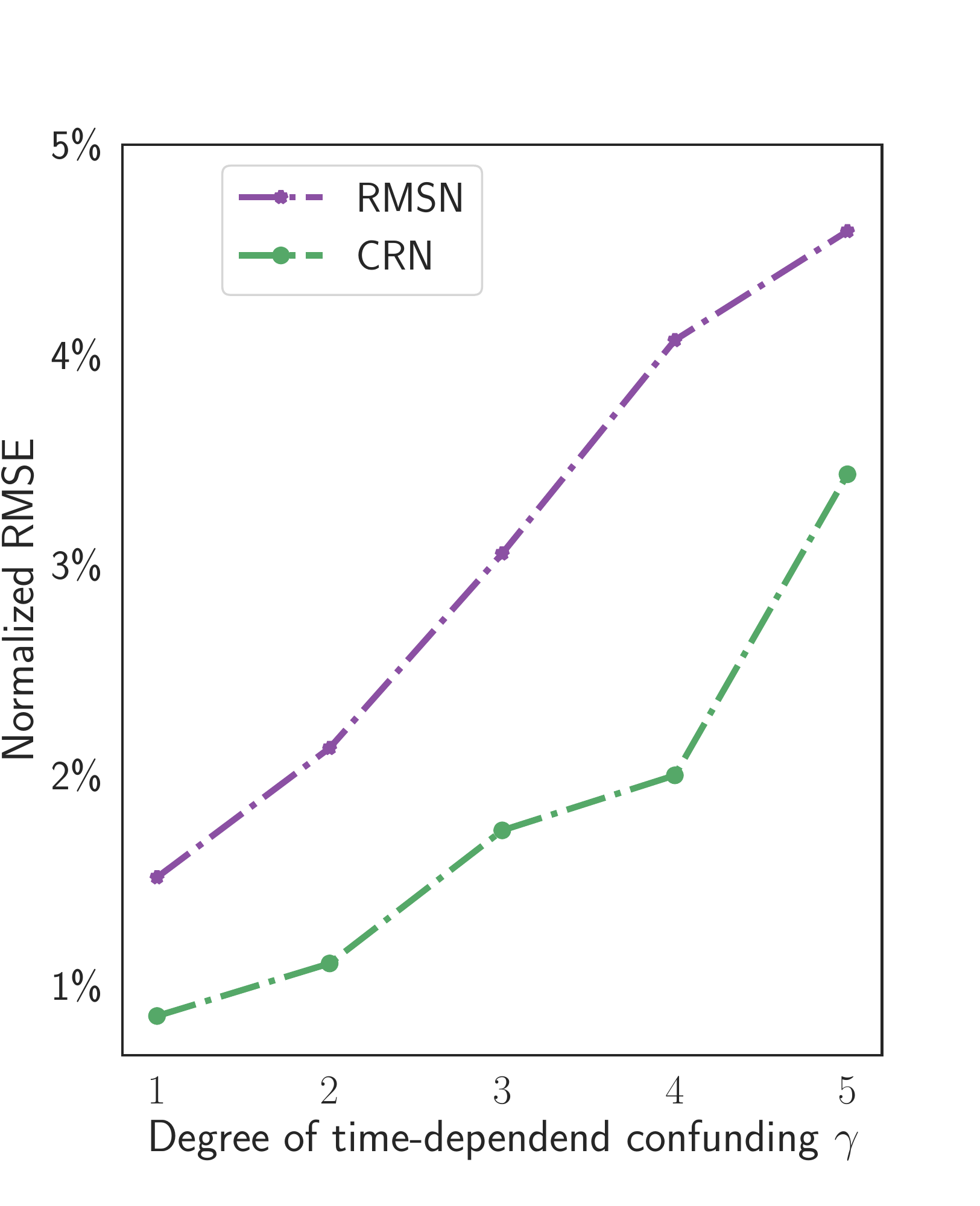} }}%
        	\caption{Results for prediction of patient counterfactuals.}%
        	\label{fig:one_step_ahead}
        	\vspace{-0.3cm}
        \end{figure}

    Figure \ref{fig:one_step_ahead} (a) shows the normalized root mean squared error (RMSE) for one-step ahead estimation of counterfactuals with varying degree of time-dependent confounding $\gamma$. The RMSE is normalized by the maximum tumour volume: $V_{max} = 1150 \text{cm}^3$.  The linear and MSM models provide a baseline for performance as they achieve the highest RMSE. While the use of IPTW in MSMs helps when $\gamma$ increases, using linear modelling has severe limitations. When there is no time-dependent confounding, the machine learning methods achieve similar performance, close to 0.6\% RMSE. As the bias in the dataset increases, the harder it becomes for the RNN and the CRN ($\lambda$ = 0) to generalize to estimate outcomes of treatments not matching the training policy. When $\gamma = 10$, CRN improves by $48.1\%$ on the same model architecture without domain adversarial training  CRN ($\lambda$ = 0).

	Our proposed model achieves the lowest RMSE across all values of $\gamma$. Compared to RMSNs, CRN improves by $\sim17\%$ when $\gamma > 6$. To highlight the gains of our method even for smaller $\gamma$,  Figure \ref{fig:one_step_ahead} (b) shows the RMSE for five-step ahead prediction (with counterfactuals generated as described in Section \ref{sec:timing} and Appendix \ref{apx:test_set_timing}). RMSNs also use a decoder for sequence prediction. However, RMSNs require training additional RNNs to estimate the IPTW, which are used to weight each sample during the decoder training. For $\tau$-step ahead prediction, IPTW involves multiplying $\tau$ weights which can result in high variance. The results in Figure \ref{fig:one_step_ahead} (b) show the problems with using IPTW to handle the time-dependent confounding bias. See Appendix \ref{apx:full_results_counterfactuals} for more results on multi-step ahead prediction. 
	
	\textbf{Balancing representation:} To evaluate whether the CRN has indeed learnt treatment invariant represenations, for $\gamma = 5$, we illustrate in Figure \ref{fig:balancing_rep} the T-SNE embeddings of the balancing representations $\Phi(\hist{H}_t)$ built by the CRN encoder for test patients. We color each point by the treatment $\vc{A}_t \in \{\text{no treatment},  \text{chemotherapy},  \text{radiotherapy}, \text{combined chemotherapy and radiotherapy}\}$ received at timestep $t$ to highlight the invariance of  $\Phi(\hist{H}_t)$ across the different treatments. In Figure \ref{fig:balancing_rep}(b), we show $\Phi(\hist{H}_t)$ only for chemotherapy and radiotherapy for better understanding.
    \vspace{-3.3mm}
    
    \begin{figure}[H]
    	\centering
    	\vspace{-0.3cm}
    	\subfloat{{\includegraphics[height=1.78cm]{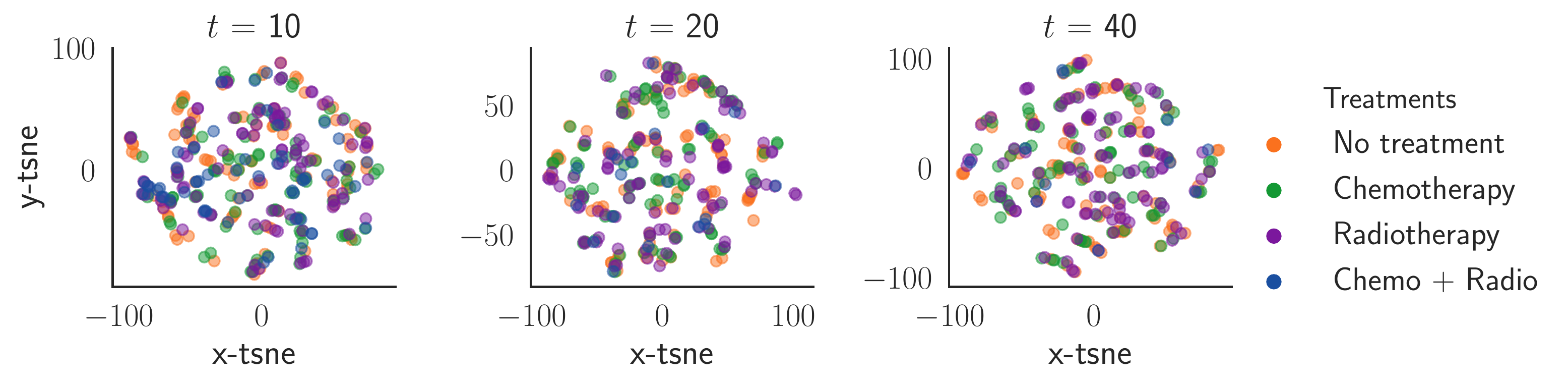} }} 
    	\quad
    	\subfloat{{\includegraphics[height=1.78cm]{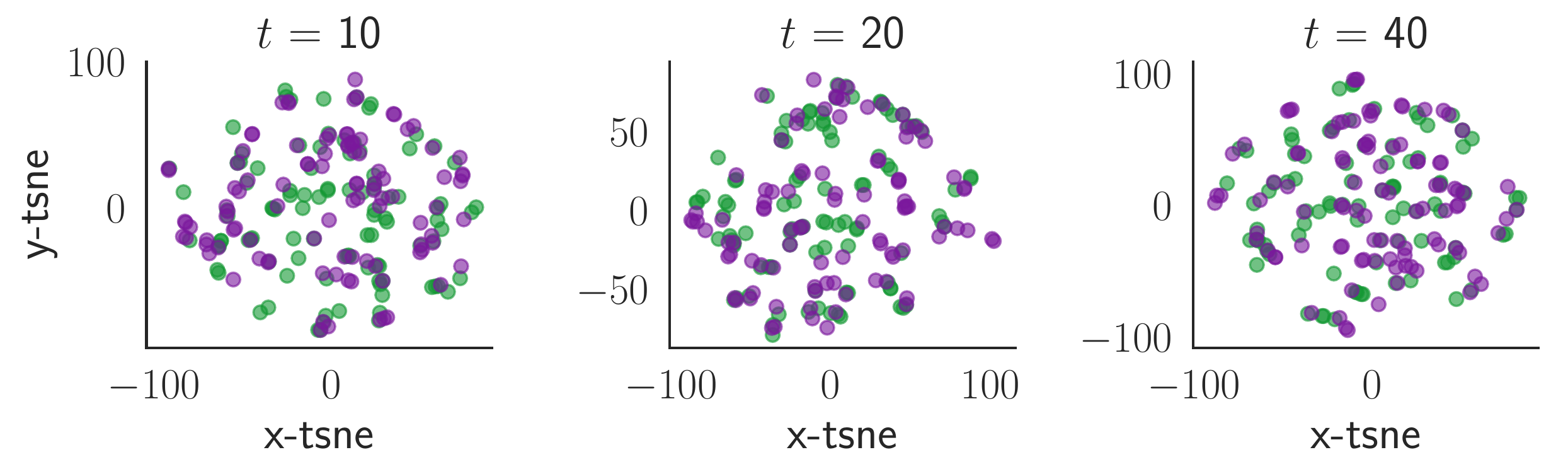} }}%
    	\caption{TSNE embedding of the balancing representation $\Phi(\hist{H}_t)$ learnt by the CRN encoder at different timesteps $t$. Notice that $\Phi(\hist{H}_t)$ is not predictive of the treatment $\mathbf{A}_t$ given at timestep $t$. }
    	\label{fig:balancing_rep}
    	\vspace{-0.3cm}
    \end{figure}

	\subsection{Evaluate recommending the right treatment and timing of treatment} \label{sec:timing}
	
	Evaluating the models just in terms of the RMSE on counterfactual estimation is also not enough for assessing their reliability when used as part of decision support systems. In this section we assess how well the models can select the correct treatment and timing of treatment for several forecasting horizons $\tau$. We generate test sets consisting of 1000 patients where for each horizon $\tau$ and for each time $t$ in a patient's trajectory, there are $\tau$ options for giving chemotherapy at one of $t,\dots t + \tau -1$ and $\tau$ options for giving radiotherapy at one of $t,\dots t + \tau -1$. At the rest of the future timesteps, no treatment is applied. These $2\tau$ treatment plans are assessed in terms of the tumour volume outcome $\vc{Y}_{t+\tau}$. We select the treatment (chemotherapy or radiotherapy) that achieves lowest $\vc{Y}_{t+\tau}$, and within the correct treatment the timing with lowest $\vc{Y}_{t+\tau}$. We also compute the normalized RMSE for predicting $\vc{Y}_{t+\tau}$. See Appendix \ref{apx:test_set_timing} for more details about the test set. The models are evaluated for 3 settings of $\gamma_c$ and $\gamma_r$.

	\begin{table}[t]
		\vspace{-0.3cm}
		\caption{Results for recommending the correct treatment and timing of treatment.}
		\label{tab:results_multi_step_prediction}
		\centering
		\setlength\tabcolsep{3.6pt}
	
		\begin{tabular}{c|c|ccc|ccc|ccc} \toprule
			\multicolumn{2}{c}{} & \multicolumn{3}{|c|}{$\gamma_c = 5$, $\gamma_r = 5$}& \multicolumn{3}{|c|}{$\gamma_c = 5$, $\gamma_r = 0$} & \multicolumn{3}{|c}{$\gamma_c = 0$, $\gamma_r = 5$}  \\ \midrule
			{} & {$\tau$} & {CRN} & {RMSN} & {MSM} & {CRN} & {RMSN} & {MSM} & {CRN} & {RMSN} & {MSM} \\ \midrule
			Normalized     & 3 &\textbf{2.43\%~} & 3.16\%~ & 6.75\%~ &\textbf{ 1.08\%}~ & 1.35\% & 3.68\% & \textbf{1.54\%~} & 1.59\%~ & 3.23\%~  \\
			RMSE			 & 4  & \textbf{2.83\%~} & 3.95\%~ & 7.65\%~ & \textbf{1.21\%}~ & 1.81\%~ & 3.84\%~ & \textbf{1.81\%~} & 2.25\%~ & 3.52\%~ \\
			& 5 &  \textbf{3.18\%~} & 4.37\%~ & 7.95\%~ & \textbf{1.33\%}~ & 2.13\%~ & 3.91\%~ &\textbf{2.03\%~} & 2.71\%~ & 3.63\%~  \\
			& 6 & \textbf{3.51\%~} & 5.61\%~ & 8.19\%~ & \textbf{1.42\%}~ & 2.41\%~ & 3.97\%~ & \textbf{2.23\%~} & 2.73\%~ & 3.71\%~ \\
			& 7 & \textbf{3.93\%~} & 6.21\%~ & 8.52\%~ & \textbf{1.53\%}~ & 2.43\%~ & 4.04\%~ & \textbf{2.43\%~} & 2.88\%~ & 3.79\%~ \\
			\midrule
			Treatment & 3 & \textbf{83.1\%~} & 75.3\%~ & 73.9\%~ & \textbf{83.2\%}~ & 78.6\%~ & 77.1\%~ & \textbf{92.9\%~} & 87.3\%~ & 74.9\%~ \\
			Accuracy  & 4 & \textbf{82.5\%~} & 74.1\%~ & 68.5\%~ &\textbf{81.3\%}~ & 77.7\%~ & 73.9\%~ & \textbf{85.7\%~} & 83.8\%~ & 74.1\%~ \\
			& 5 & \textbf{73.5\%~} & 72.7\%~ & 63.2\%~ & \textbf{78.3\%}~ & 77.2\%~ & 72.3\%~ & \textbf{83.8\%~} & 82.1\%~ & 72.8\%~ \\
			& 6 & \textbf{69.4\%~} & 66.7\%~ & 62.7\%~ & \textbf{79.5\%}~ & 76.3\%~ & 71.8\%~ & \textbf{78.6\%~} & 69.7\%~ & 64.5\%~ \\
			& 7 & \textbf{71.2\%~} & 68.8\%~ & 62.4\%~ & \textbf{72.7\%}~ & 71.8\%~ & 71.6\%~ & \textbf{71.9\%~} & 69.3\%~ & 61.2\%~ \\
			\midrule
			Treatment    & 3 & \textbf{79.6\%~} & 78.1\%~ & 67.6\%~ & \textbf{80.5\%}~ & 76.8\%~ &  77.5\%~ & \textbf{79.8\%~} & 75.7\%~ & 60.6\%~ \\
			Timing  	& 4 & \textbf{73.9\%~} & 70.3\%~ & 63.1\%~ &\textbf{79.0\%}~ & 77.2\%~ & 73.4\%~ & \textbf{75.4\%~} & 71.4\%~ & 58.2\%~ \\
			Accuracy	& 5 & \textbf{69.8\%~} & 68.6\%~ & 62.4\%~ & \textbf{78.3\%}~ & 73.3\%~ & 63.6\%~ & \textbf{66.9\%~} & 31.3\%~ & 29.5\%~ \\
			& 6 & \textbf{66.9\%~} & 66.2\%~ & 62.6\%~ & \textbf{73.5\%}~ & 72.1\%~ & 63.9\%~ & \textbf{65.8\%~} & 24.2\%~ & 15.5\%~ \\
			& 7 & \textbf{64.5\%~} & 63.6\%~ & 62.2\%~ & \textbf{70.6\%}~ & 57.4\%~ & 44.2\%~ & \textbf{63.9\%~} & 25.6\%~ & 12.5\%~ \\
			\bottomrule
		\end{tabular}
	\vspace{-0.3cm}
	\end{table}

	Table \ref{tab:results_multi_step_prediction} shows the results for this evaluation set-up. The treatment accuracy denotes the percentage of patients for which the correct treatment was selected, while the treatment timing accuracy is the percentage for which the correct timing was selected. Note that when $\gamma_c=0$ and $\gamma_r=5$, RMSN and MSM select the wrong treatment timing for projection horizons $\tau > 4$. CRN performs similarly among the different policies present in the observational data and achieve the lowest RMSE and highest accuracy in selecting the correct treatment and timing of treatment.

    In Appendix \ref{apx:mimic} we also show the applicability of the CRN in more complex medical scenarios involving real data. We provide experimental results based on the Medical Information Mart for Intensive Care (MIMIC III) database \citep{johnson2016mimic} consisting of electronic health records from patients in the ICU.

	\section{Conclusion}
	
    Despite its wide applicability, the problem of causal inference for time-dependent treatments has been relatively less studied compared to problem of causal inference in the static setting. Both new methods and theory are necessary to be able to harness the full potential of observational data for learning individualized effects of complex treatment scenarios. Further work in this direction is needed for proposing alternative methods for handling time-dependent confounders, for modelling combinations of treatments assigned over time or for  estimating the individualized effects of time-dependent treatments with associated dosage. 

	In this paper, we introduced the Counterfactual Recurrent Network (CRN), a model that estimates individualized effects of treatments over time using a novel way of handling the bias from time-dependent confounders through adversarial training.
	Using a model of tumour growth, we validated CRN in realistic medical scenarios and we showed improvements over existing state-of-the-art methods. We also showed the applicability of the CRN a real dataset consiting of patient electronic health records. The counterfactual predictions of CRN have the potential to be used as part of clinical decision support systems to address relevant medical challenges involving selecting the best treatments for patients over time, identify optimal treatment timings but also when the treatment is no longer needed. In future work, we will aim to build better  balancing representations and to provide theoretical guarantees for the expected error on the counterfactuals. 

\newpage

\section*{Acknowledgments}
We would like to thank the reviewers for their valuable feedback. The research presented in this paper was supported by The Alan Turing Institute, under the EPSRC grant EP/N510129/1 and by the US Office of Naval Research (ONR).

\bibliography{iclr2020_conference}

\begin{thebibliography}{62}
\providecommand{\natexlab}[1]{#1}
\providecommand{\url}[1]{\texttt{#1}}
\expandafter\ifx\csname urlstyle\endcsname\relax
  \providecommand{\doi}[1]{doi: #1}\else
  \providecommand{\doi}{doi: \begingroup \urlstyle{rm}\Url}\fi

\bibitem[Abadie \& Imbens(2016)Abadie and Imbens]{abadie2016matching}
Alberto Abadie and Guido~W Imbens.
\newblock Matching on the estimated propensity score.
\newblock \emph{Econometrica}, 84\penalty0 (2):\penalty0 781--807, 2016.

\bibitem[Alaa \& van~der Schaar(2018)Alaa and van~der Schaar]{alaa2018limits}
Ahmed Alaa and Mihaela van~der Schaar.
\newblock Limits of estimating heterogeneous treatment effects: Guidelines for
  practical algorithm design.
\newblock In \emph{International Conference on Machine Learning}, pp.\
  129--138, 2018.

\bibitem[Alaa \& van~der Schaar(2017)Alaa and van~der Schaar]{alaa2017bayesian}
Ahmed~M Alaa and Mihaela van~der Schaar.
\newblock Bayesian inference of individualized treatment effects using
  multi-task gaussian processes.
\newblock In \emph{Advances in Neural Information Processing Systems}, pp.\
  3424--3432, 2017.

\bibitem[Ali et~al.(2019)Ali, Naureen, Tariq, Farrukh, Usman, Khattak, and
  Ahsan]{ali2019rational}
Muhammad Ali, Humaira Naureen, Muhammad~Haseeb Tariq, Muhammad~Junaid Farrukh,
  Abubakar Usman, Shahana Khattak, and Hina Ahsan.
\newblock Rational use of antibiotics in an intensive care unit: a
  retrospective study of the impact on clinical outcomes and mortality rate.
\newblock \emph{Infection and Drug Resistance}, 12:\penalty0 493, 2019.

\bibitem[Arjas \& Parner(2004)Arjas and Parner]{arjas2004causal}
Elja Arjas and Jan Parner.
\newblock Causal reasoning from longitudinal data.
\newblock \emph{Scandinavian Journal of Statistics}, 31\penalty0 (2):\penalty0
  171--187, 2004.

\bibitem[Atan et~al.(2018)Atan, Zame, and van~der Schaar]{atan2018learning}
Onur Atan, William~R Zame, and Mihaela van~der Schaar.
\newblock Learning optimal policies from observational data.
\newblock \emph{International Conference on Machine Learning CausalML
  workshop}, 2018.

\bibitem[Austin(2011)]{austin2011introduction}
Peter~C Austin.
\newblock An introduction to propensity score methods for reducing the effects
  of confounding in observational studies.
\newblock \emph{Multivariate behavioral research}, 46\penalty0 (3):\penalty0
  399--424, 2011.

\bibitem[Bartsch et~al.(2007)Bartsch, Dally, Popanda, Risch, and
  Schmezer]{bartsch2007genetic}
Helmut Bartsch, Heike Dally, Odilia Popanda, Angela Risch, and Peter Schmezer.
\newblock Genetic risk profiles for cancer susceptibility and therapy response.
\newblock In \emph{Cancer Prevention}, pp.\  19--36. Springer, 2007.

\bibitem[Ben-David et~al.(2007)Ben-David, Blitzer, Crammer, and
  Pereira]{ben2007analysis}
Shai Ben-David, John Blitzer, Koby Crammer, and Fernando Pereira.
\newblock Analysis of representations for domain adaptation.
\newblock In \emph{Advances in neural information processing systems}, pp.\
  137--144, 2007.

\bibitem[Bengio et~al.(2012)Bengio, Courville, and
  Vincent]{bengio2012representation}
Y~Bengio, A~Courville, and P~Vincent.
\newblock Representation learning: a review and new perspectives. arxiv. org.
\newblock 2012.

\bibitem[Booth \& Tannock(2014)Booth and Tannock]{booth2014randomised}
CM~Booth and IF~Tannock.
\newblock Randomised controlled trials and population-based observational
  research: partners in the evolution of medical evidence.
\newblock \emph{British journal of cancer}, 110\penalty0 (3):\penalty0 551,
  2014.

\bibitem[De~Bus et~al.(2018)De~Bus, Gadeyne, Steen, Boelens, Claeys, Benoit,
  De~Waele, Decruyenaere, and Depuydt]{de2018complete}
Liesbet De~Bus, Bram Gadeyne, Johan Steen, Jerina Boelens, Geert Claeys,
  Dominique Benoit, Jan De~Waele, Johan Decruyenaere, and Pieter Depuydt.
\newblock A complete and multifaceted overview of antibiotic use and infection
  diagnosis in the intensive care unit: results from a prospective four-year
  registration.
\newblock \emph{Critical Care}, 22\penalty0 (1):\penalty0 241, 2018.

\bibitem[Doroudi et~al.(2017)Doroudi, Thomas, and
  Brunskill]{doroudi2017importance}
Shayan Doroudi, Philip~S Thomas, and Emma Brunskill.
\newblock Importance sampling for fair policy selection.
\newblock \emph{Grantee Submission}, 2017.

\bibitem[Gal \& Ghahramani(2016)Gal and Ghahramani]{gal2016theoretically}
Yarin Gal and Zoubin Ghahramani.
\newblock A theoretically grounded application of dropout in recurrent neural
  networks.
\newblock In \emph{Advances in neural information processing systems}, pp.\
  1019--1027, 2016.

\bibitem[Ganin et~al.(2016)Ganin, Ustinova, Ajakan, Germain, Larochelle,
  Laviolette, Marchand, and Lempitsky]{ganin2016domain}
Yaroslav Ganin, Evgeniya Ustinova, Hana Ajakan, Pascal Germain, Hugo
  Larochelle, Fran{\c{c}}ois Laviolette, Mario Marchand, and Victor Lempitsky.
\newblock Domain-adversarial training of neural networks.
\newblock \emph{The Journal of Machine Learning Research}, 17\penalty0
  (1):\penalty0 2096--2030, 2016.

\bibitem[Geng et~al.(2017)Geng, Paganetti, and Grassberger]{geng2017prediction}
Changran Geng, Harald Paganetti, and Clemens Grassberger.
\newblock Prediction of treatment response for combined chemo-and radiation
  therapy for non-small cell lung cancer patients using a bio-mathematical
  model.
\newblock \emph{Scientific reports}, 7\penalty0 (1):\penalty0 13542, 2017.

\bibitem[Guo et~al.(2017)Guo, Thomas, and Brunskill]{guo2017using}
Zhaohan Guo, Philip~S Thomas, and Emma Brunskill.
\newblock Using options and covariance testing for long horizon off-policy
  policy evaluation.
\newblock In \emph{Advances in Neural Information Processing Systems}, pp.\
  2492--2501, 2017.

\bibitem[Hallak et~al.(2015)Hallak, Schnitzler, Mann, and
  Mannor]{hallak2015off}
Assaf Hallak, Fran{\c{c}}ois Schnitzler, Timothy Mann, and Shie Mannor.
\newblock Off-policy model-based learning under unknown factored dynamics.
\newblock In \emph{International Conference on Machine Learning}, pp.\
  711--719, 2015.

\bibitem[Hern{\'a}n et~al.(2001)Hern{\'a}n, Brumback, and
  Robins]{hernan2001marginal}
Miguel~A Hern{\'a}n, Babette Brumback, and James~M Robins.
\newblock Marginal structural models to estimate the joint causal effect of
  nonrandomized treatments.
\newblock \emph{Journal of the American Statistical Association}, 96\penalty0
  (454):\penalty0 440--448, 2001.

\bibitem[Hern{\'a}n et~al.(2000)Hern{\'a}n, Brumback, and
  Robins]{hernan2000marginal}
Miguel~{\'A}ngel Hern{\'a}n, Babette Brumback, and James~M Robins.
\newblock Marginal structural models to estimate the causal effect of
  zidovudine on the survival of hiv-positive men.
\newblock \emph{Epidemiology}, pp.\  561--570, 2000.

\bibitem[Hochreiter \& Schmidhuber(1997)Hochreiter and
  Schmidhuber]{hochreiter1997long}
Sepp Hochreiter and J{\"u}rgen Schmidhuber.
\newblock Long short-term memory.
\newblock \emph{Neural computation}, 9\penalty0 (8):\penalty0 1735--1780, 1997.

\bibitem[Hoiles \& Van Der~Schaar(2016)Hoiles and Van
  Der~Schaar]{hoiles2016non}
William Hoiles and Mihaela Van Der~Schaar.
\newblock A non-parametric learning method for confidently estimating patient's
  clinical state and dynamics.
\newblock In \emph{Advances in Neural Information Processing Systems}, pp.\
  2020--2028, 2016.

\bibitem[Howe et~al.(2012)Howe, Cole, Mehta, and Kirk]{howe2012estimating}
Chanelle~J Howe, Stephen~R Cole, Shruti~H Mehta, and Gregory~D Kirk.
\newblock Estimating the effects of multiple time-varying exposures using joint
  marginal structural models: alcohol consumption, injection drug use, and hiv
  acquisition.
\newblock \emph{Epidemiology (Cambridge, Mass.)}, 23\penalty0 (4):\penalty0
  574, 2012.

\bibitem[Imai \& Ratkovic(2014)Imai and Ratkovic]{imai2014covariate}
Kosuke Imai and Marc Ratkovic.
\newblock Covariate balancing propensity score.
\newblock \emph{Journal of the Royal Statistical Society: Series B (Statistical
  Methodology)}, 76\penalty0 (1):\penalty0 243--263, 2014.

\bibitem[Imai \& Van~Dyk(2004)Imai and Van~Dyk]{imai2004causal}
Kosuke Imai and David~A Van~Dyk.
\newblock Causal inference with general treatment regimes: Generalizing the
  propensity score.
\newblock \emph{Journal of the American Statistical Association}, 99\penalty0
  (467):\penalty0 854--866, 2004.

\bibitem[Jiang \& Li(2015)Jiang and Li]{jiang2015doubly}
Nan Jiang and Lihong Li.
\newblock Doubly robust off-policy value evaluation for reinforcement learning.
\newblock \emph{arXiv preprint arXiv:1511.03722}, 2015.

\bibitem[Johansson et~al.(2016)Johansson, Shalit, and
  Sontag]{johansson2016learning}
Fredrik Johansson, Uri Shalit, and David Sontag.
\newblock Learning representations for counterfactual inference.
\newblock In \emph{International conference on machine learning}, pp.\
  3020--3029, 2016.

\bibitem[Johnson et~al.(2016)Johnson, Pollard, Shen, Li-wei, Feng, Ghassemi,
  Moody, Szolovits, Celi, and Mark]{johnson2016mimic}
Alistair~EW Johnson, Tom~J Pollard, Lu~Shen, H~Lehman Li-wei, Mengling Feng,
  Mohammad Ghassemi, Benjamin Moody, Peter Szolovits, Leo~Anthony Celi, and
  Roger~G Mark.
\newblock Mimic-iii, a freely accessible critical care database.
\newblock \emph{Scientific data}, 3:\penalty0 160035, 2016.

\bibitem[Kifer et~al.(2004)Kifer, Ben-David, and Gehrke]{kifer2004detecting}
Daniel Kifer, Shai Ben-David, and Johannes Gehrke.
\newblock Detecting change in data streams.
\newblock In \emph{Proceedings of the Thirtieth international conference on
  Very large data bases-Volume 30}, pp.\  180--191. VLDB Endowment, 2004.

\bibitem[Kingma \& Ba(2014)Kingma and Ba]{kingma2014adam}
Diederik~P Kingma and Jimmy Ba.
\newblock Adam: A method for stochastic optimization.
\newblock \emph{arXiv preprint arXiv:1412.6980}, 2014.

\bibitem[Li \& Fu(2017)Li and Fu]{li2017matching}
Sheng Li and Yun Fu.
\newblock Matching on balanced nonlinear representations for treatment effects
  estimation.
\newblock In \emph{Advances in Neural Information Processing Systems}, pp.\
  929--939, 2017.

\bibitem[Li et~al.(2018)Li, Tian, Gong, Liu, Liu, Zhang, and Tao]{li2018deep}
Ya~Li, Xinmei Tian, Mingming Gong, Yajing Liu, Tongliang Liu, Kun Zhang, and
  Dacheng Tao.
\newblock Deep domain generalization via conditional invariant adversarial
  networks.
\newblock In \emph{Proceedings of the European Conference on Computer Vision
  (ECCV)}, pp.\  624--639, 2018.

\bibitem[Lim et~al.(2018)Lim, Alaa, and van~der Schaar]{lim2018forecasting}
Bryan Lim, Ahmed Alaa, and Mihaela van~der Schaar.
\newblock Forecasting treatment responses over time using recurrent marginal
  structural networks.
\newblock In \emph{Advances in Neural Information Processing Systems}, pp.\
  7493--7503, 2018.

\bibitem[Lok et~al.(2008)]{lok2008statistical}
Judith~J Lok et~al.
\newblock Statistical modeling of causal effects in continuous time.
\newblock \emph{The Annals of Statistics}, 36\penalty0 (3):\penalty0
  1464--1507, 2008.

\bibitem[Mansournia et~al.(2012)Mansournia, Danaei, Forouzanfar, Mahmoodi,
  Jamali, Mansournia, and Mohammad]{mansournia2012effect}
Mohammad~Ali Mansournia, Goodarz Danaei, Mohammad~Hossein Forouzanfar, Mahmood
  Mahmoodi, Mohsen Jamali, Nasrin Mansournia, and Kazem Mohammad.
\newblock Effect of physical activity on functional performance and knee pain
  in patients with osteoarthritis: analysis with marginal structural models.
\newblock \emph{Epidemiology}, pp.\  631--640, 2012.

\bibitem[Mansournia et~al.(2017)Mansournia, Etminan, Danaei, Kaufman, and
  Collins]{mansournia2017handling}
Mohammad~Ali Mansournia, Mahyar Etminan, Goodarz Danaei, Jay~S Kaufman, and
  Gary Collins.
\newblock Handling time varying confounding in observational research.
\newblock \emph{bmj}, 359:\penalty0 j4587, 2017.

\bibitem[Mortimer et~al.(2005)Mortimer, Neugebauer, Van Der~Laan, and
  Tager]{mortimer2005application}
Kathleen~M Mortimer, Romain Neugebauer, Mark Van Der~Laan, and Ira~B Tager.
\newblock An application of model-fitting procedures for marginal structural
  models.
\newblock \emph{American Journal of Epidemiology}, 162\penalty0 (4):\penalty0
  382--388, 2005.

\bibitem[Neil et~al.(2016)Neil, Pfeiffer, and Liu]{neil2016phased}
Daniel Neil, Michael Pfeiffer, and Shih-Chii Liu.
\newblock Phased lstm: Accelerating recurrent network training for long or
  event-based sequences.
\newblock In \emph{Advances in Neural Information Processing Systems}, pp.\
  3882--3890, 2016.

\bibitem[Neyman(1923)]{neyman1923applications}
Jersey Neyman.
\newblock Sur les applications de la th{\'e}orie des probabilit{\'e}s aux
  experiences agricoles: Essai des principes.
\newblock \emph{Roczniki Nauk Rolniczych}, 10:\penalty0 1--51, 1923.

\bibitem[P{\u{a}}duraru et~al.(2013)P{\u{a}}duraru, Precup, Pineau, and
  Com{\u{a}}nici]{puaduraru2013empirical}
Cosmin P{\u{a}}duraru, Doina Precup, Joelle Pineau, and Gheorghe
  Com{\u{a}}nici.
\newblock An empirical analysis of off-policy learning in discrete mdps.
\newblock In \emph{European Workshop on Reinforcement Learning}, pp.\  89--102,
  2013.

\bibitem[Pearl et~al.(2009)]{pearl2009causal}
Judea Pearl et~al.
\newblock Causal inference in statistics: An overview.
\newblock \emph{Statistics surveys}, 3:\penalty0 96--146, 2009.

\bibitem[Platt et~al.(2009)Platt, Schisterman, and Cole]{platt2009time}
Robert~W Platt, Enrique~F Schisterman, and Stephen~R Cole.
\newblock Time-modified confounding.
\newblock \emph{American journal of epidemiology}, 170\penalty0 (6):\penalty0
  687--694, 2009.

\bibitem[Precup(2000)]{precup2000eligibility}
Doina Precup.
\newblock Eligibility traces for off-policy policy evaluation.
\newblock \emph{Computer Science Department Faculty Publication Series}, pp.\
  ~80, 2000.

\bibitem[Robins(1986)]{robins1986new}
James Robins.
\newblock A new approach to causal inference in mortality studies with a
  sustained exposure period—application to control of the healthy worker
  survivor effect.
\newblock \emph{Mathematical modelling}, 7\penalty0 (9-12):\penalty0
  1393--1512, 1986.

\bibitem[Robins(1994)]{robins1994correcting}
James~M Robins.
\newblock Correcting for non-compliance in randomized trials using structural
  nested mean models.
\newblock \emph{Communications in Statistics-Theory and methods}, 23\penalty0
  (8):\penalty0 2379--2412, 1994.

\bibitem[Robins(1999)]{robins1999association}
James~M Robins.
\newblock Association, causation, and marginal structural models.
\newblock \emph{Synthese}, 121\penalty0 (1):\penalty0 151--179, 1999.

\bibitem[Robins \& Hern{\'a}n(2008)Robins and Hern{\'a}n]{robins2008estimation}
James~M Robins and Miguel~A Hern{\'a}n.
\newblock Estimation of the causal effects of time-varying exposures.
\newblock In \emph{Longitudinal data analysis}, pp.\  547--593. Chapman and
  Hall/CRC, 2008.

\bibitem[Robins et~al.(2000)Robins, Hernan, and Brumback]{robins2000marginal}
James~M Robins, Miguel~Angel Hernan, and Babette Brumback.
\newblock Marginal structural models and causal inference in epidemiology,
  2000.

\bibitem[Roy et~al.(2016)Roy, Lum, and Daniels]{roy2016bayesian}
Jason Roy, Kirsten~J Lum, and Michael~J Daniels.
\newblock A bayesian nonparametric approach to marginal structural models for
  point treatments and a continuous or survival outcome.
\newblock \emph{Biostatistics}, 18\penalty0 (1):\penalty0 32--47, 2016.

\bibitem[Rubin(1978)]{rubin1978bayesian}
Donald~B Rubin.
\newblock Bayesian inference for causal effects: The role of randomization.
\newblock \emph{The Annals of statistics}, pp.\  34--58, 1978.

\bibitem[Schisterman et~al.(2009)Schisterman, Cole, and
  Platt]{schisterman2009overadjustment}
Enrique~F Schisterman, Stephen~R Cole, and Robert~W Platt.
\newblock Overadjustment bias and unnecessary adjustment in epidemiologic
  studies.
\newblock \emph{Epidemiology (Cambridge, Mass.)}, 20\penalty0 (4):\penalty0
  488, 2009.

\bibitem[Schulam \& Saria(2017)Schulam and Saria]{schulam2017reliable}
Peter Schulam and Suchi Saria.
\newblock Reliable decision support using counterfactual models.
\newblock In \emph{Advances in Neural Information Processing Systems}, pp.\
  1697--1708, 2017.

\bibitem[Sebag et~al.(2019)Sebag, Heinrich, Schoenauer, Sebag, Wu, and
  Altschuler]{sebag2019multi}
Alice~Schoenauer Sebag, Louise Heinrich, Marc Schoenauer, Mich{\`e}le Sebag,
  Lani Wu, and Steven Altschuler.
\newblock Multi-domain adversarial learning.
\newblock In \emph{ICLR'19-Seventh annual International Conference on Learning
  Representations}, 2019.

\bibitem[Shalit et~al.(2017)Shalit, Johansson, and
  Sontag]{shalit2017estimating}
Uri Shalit, Fredrik~D Johansson, and David Sontag.
\newblock Estimating individual treatment effect: generalization bounds and
  algorithms.
\newblock In \emph{Proceedings of the 34th International Conference on Machine
  Learning-Volume 70}, pp.\  3076--3085. JMLR. org, 2017.

\bibitem[Soleimani et~al.(2017)Soleimani, Subbaswamy, and
  Saria]{soleimani2017treatment}
Hossein Soleimani, Adarsh Subbaswamy, and Suchi Saria.
\newblock Treatment-response models for counterfactual reasoning with
  continuous-time, continuous-valued interventions.
\newblock \emph{arXiv preprint arXiv:1704.02038}, 2017.

\bibitem[Swaminathan \& Joachims(2015{\natexlab{a}})Swaminathan and
  Joachims]{swaminathan2015batch}
Adith Swaminathan and Thorsten Joachims.
\newblock Batch learning from logged bandit feedback through counterfactual
  risk minimization.
\newblock \emph{Journal of Machine Learning Research}, 16\penalty0
  (1):\penalty0 1731--1755, 2015{\natexlab{a}}.

\bibitem[Swaminathan \& Joachims(2015{\natexlab{b}})Swaminathan and
  Joachims]{swaminathan2015self}
Adith Swaminathan and Thorsten Joachims.
\newblock The self-normalized estimator for counterfactual learning.
\newblock In \emph{advances in neural information processing systems}, pp.\
  3231--3239, 2015{\natexlab{b}}.

\bibitem[Thomas et~al.(2015)Thomas, Theocharous, and
  Ghavamzadeh]{thomas2015high}
Philip~S Thomas, Georgios Theocharous, and Mohammad Ghavamzadeh.
\newblock High-confidence off-policy evaluation.
\newblock In \emph{Twenty-Ninth AAAI Conference on Artificial Intelligence},
  2015.

\bibitem[Waheed et~al.(2003)Waheed, Williams, Brett, Baldock, and
  Soni]{waheed2003white}
U~Waheed, P~Williams, S~Brett, G~Baldock, and N~Soni.
\newblock White cell count and intensive care unit outcome.
\newblock \emph{Anaesthesia}, 58\penalty0 (2):\penalty0 180--182, 2003.

\bibitem[Xu et~al.(2016)Xu, Xu, and Saria]{xu2016bayesian}
Yanbo Xu, Yanxun Xu, and Suchi Saria.
\newblock A bayesian nonparametric approach for estimating individualized
  treatment-response curves.
\newblock In \emph{Machine Learning for Healthcare Conference}, pp.\  282--300,
  2016.

\bibitem[Yao et~al.(2018)Yao, Li, Li, Huai, Gao, and
  Zhang]{yao2018representation}
Liuyi Yao, Sheng Li, Yaliang Li, Mengdi Huai, Jing Gao, and Aidong Zhang.
\newblock Representation learning for treatment effect estimation from
  observational data.
\newblock In \emph{Advances in Neural Information Processing Systems}, pp.\
  2633--2643, 2018.

\bibitem[Yoon et~al.(2018)Yoon, Jordon, and van~der Schaar]{yoon2018ganite}
Jinsung Yoon, James Jordon, and Mihaela van~der Schaar.
\newblock Ganite: Estimation of individualized treatment effects using
  generative adversarial nets.
\newblock \emph{International Conference on Learning Representations (ICLR)},
  2018.

\end{thebibliography}
\bibliographystyle{iclr2020_conference}

\newpage
\appendix 

\subsection*{Appendix}

\section{Extended related work} \label{apx:related_work}
	
	\textbf{Causal inference in the static setting:} A large number of methods have been proposed to learn treatment effects from observational data in the static setting. In this case, it is needed to adjust for the selection bias; bias caused by the fact that, in the observational dataset, the treatment assignments depend on the patient features. Several ways of handling the selection bias involve using propensity matching \citep{austin2011introduction, imai2014covariate, abadie2016matching}, building representations where treated and un-treated populations had similar distributions \citep{johansson2016learning, shalit2017estimating, li2017matching, yao2018representation} or performing propensity-aware hyperparameter tuning \citep{alaa2017bayesian, alaa2018limits}. However, these methods for the static setting cannot be extended directly to time-varying treatments \citep{hernan2000marginal, schisterman2009overadjustment}.
	
	\textbf{Learning optimal policies:} A related problem to ours involves learning the optimal treatment policies from logged data \citep{swaminathan2015batch, swaminathan2015self, atan2018learning}. That is, learning the treatment option that would give the best reward. Note the difference to the causal inference setting considered in this paper, where the aim is to learn the counterfactual patient outcomes under all possible treatment options. Learning all of the counterfactual outcomes is a harder problem and can also be used for finding the optimal treatment. 
	
	A method for learning optimal policies, proposed by \cite{atan2018learning} uses domain adversarial training to build a representation that is invariant to the following two domains: observational data and simulated randomized clinical trial data, where the treatments have equal probabilities. \cite{atan2018learning} only considers the static setting and aims to choose the optimal treatment instead of estimating all of the counterfactual outcomes. In our paper the aim is to eliminate the bias from the time-dependent confounders and reliably estimate \emph{all of the potential outcomes}; thus, at each timestep $t$ we build a representation that is invariant to the treatment.
	
	\textbf{Off-policy evaluation in reinforcement learning:} In reinforcement learning, a similar problem to ours is off-policy evaluation, which uses retrospective observational data, also known as logged bandit feedback \citep{hoiles2016non, puaduraru2013empirical, doroudi2017importance}. In this case, the retrospective observational data consists of sequences of states, actions and rewards which were generated by an agent operating under an unknown policy. The off-policy evaluation methods aim to use this data to estimate the expected reward of a target policy. These methods use algorithms based on importance sampling \citep{precup2000eligibility, thomas2015high, guo2017using}, action-value function approximation (model based) \citep{hallak2015off} or doubly robust combination of both approaches \citep{jiang2015doubly}. Nevertheless, these methods focus on obtaining average rewards of policies, while in our case the aim is to estimate individualized patient outcomes for future treatments.

	\section{Assumptions} \label{apx:assumptions}
	
	The standard assumptions needed for identifying the treatment effects are \citep{robins2008estimation, lim2018forecasting, schulam2017reliable}: 

	\textbf{Assumption 1: Consistency}.  If $\vc{A}_t = \vc{a}_t$ for a given patient, then the potential outcome for treatment $\vc{a}_t$ is the same as the observed (factual) outcome: $\vc{Y}_{t+1}[\vc{a}_t] = \vc{Y}_{t+1}$. 
		
	\textbf{Assumption 2: Positivity (Overlap)} \citep{imai2004causal}: 
		If $P( \hist{A}_{t-1} = \hist{a}_{t-1}, \hist{X}_t = \hist{x}_t)\neq 0$ then $P(\mathbf{A}_t = \mathbf{a}_t \mid \hist{A}_{t-1} = \hist{a}_{t-1}, \hist{X}_t = \hist{x}_t) > 0$ for all $\hist{a}_t$. 
		
	\textbf{Assumption 3: Sequential strong ignorability.} $\vc{Y}_{t+1}[\vc{a}_t]  \ci \vc{A}_t \mid \hist{A}_{t-1},  \hist{X}_{t}, \forall \vc{a}_t \in \cl{A}, \forall t.$ 
		
	Assumption 2 means that, for each timestep, each treatment has non-zero probability of being assigned. Assumption 3 means that there are no hidden confounders, that is, all of covariates affecting both the treatment assignment and the outcomes are present in the the observational dataset. Note that while assumption 3 is standard across all methods for estimating treatment effects, it is not testable in practice \citep{robins2000marginal,  pearl2009causal}.
	
	\clearpage
	
	\section{Time-dependent Confounding} \label{apx:time_dependent_confounding}

	Figure \ref{fig:time-dependent_confounding} illustrates the causal graphs for a time-varying exposures with 2-steps \citep{robins2000marginal}. In Figure \ref{fig:time-dependent_confounding} (a), the covariate $X$ is a time-dependent confounder because it affects the treatment assignments and at the same time, its value is changed by past treatments \citep{mansournia2017handling}, as illustrated by the red arrows. Thus, the treatment probabilities at each time $t$ depend on the history of covariate $X$ and past treatments. Note that $U_0$ and $U_1$ are hidden variables which only affect the covariates, i.e. they do not have arrows into the treatments. Thus, the no hidden confounders assumption (Assumption 3) is satisfied.  
	
	Figure \ref{fig:time-dependent_confounding} (a) and (b) illustrate the two cases when there is no bias from time-dependent confounding. In Figure \ref{fig:time-dependent_confounding} (a) the treatment probabilities are independent, while in Figure \ref{fig:time-dependent_confounding} (b) they depend on past treatments. 
	
	\begin{figure}[H]
		\centering
		\includegraphics[width=\columnwidth]{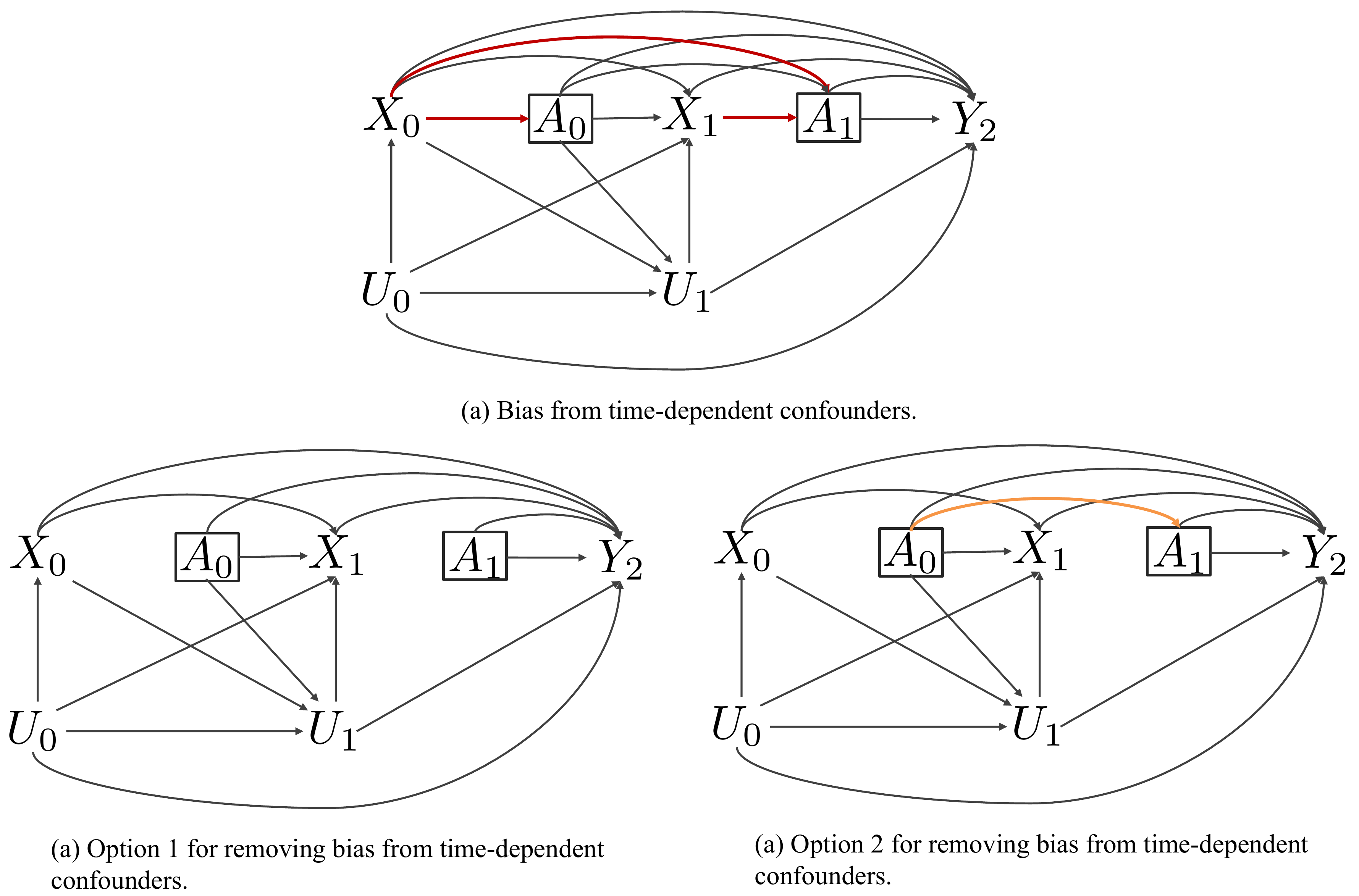}
		\caption{Causal graphs for 2-step time-varying exposures \citep{robins2000marginal}. $X_0, X_1$ are patient covariates, $A_0, A_1$ are treatments, $U_0$, $U_1$ are unobserved variable and $Y_2$ is the outcome.}
		\label{fig:time-dependent_confounding}
	\end{figure}

	\textbf{Marginal Structural Models} \cite{robins2000marginal}.
	To remove the association between time-dependent confounders and time-varying treatments, Marginal Structural Models propose using inverse probability of treatment weighting (IPTW). Without loss of generality, consider the use of MSMs with univariate treatments, baseline variables and outcomes. The outcome after $t$ timesteps is parametrized as follows: $\mathbf{E}[Y_{t+1} \mid \vc{a}_1, \dots \vc{a}_t, V] = g(\vc{a}_1, \dots \vc{a}_n, V; \theta)$, where $g(\cdot)$ is usually a linear function with parameters $\theta$. To remove the bias from the time-dependent confounders present in the observational dataset, in the regression model $g(\cdot)$ MSMs weights each patients using either stabilized weights:
	\begin{equation}
	SW(t) = \prod_{l=1}^{t} \dfrac{f(\vc{A}_l \mid \hist{A}_{l-1})}{ f(\vc{A}_l \mid \hist{X}_{l}, \hist{A}_{l-1}, \vc{V})}
	\end{equation}
	or unstabilized weights: 
	\begin{equation}
	W(t) = \prod_{l=1}^{t} \dfrac{1}{f(\vc{A}_l \mid \hist{X}_{l}, \hist{A}_{l-1}, \vc{V})}, 
	\end{equation}
	where $f(\cdot)$ represents the conditional probability mass function for discrete treatments.

	Inverse probability of treatment weighting (IPTW) creates a pseudo-population where each member consists of themselves and $W-1$ (or $SW-1$) copies added though weighting. In this pseudo-population, Robins \cite{robins1999association} shows that $\hist{X}_t$ does not predict treatment $\vc{A}_t$, thus removing the bias from time-dependent confounders. 
	
	When using unstabilized weights $W$, the causal graph in the pseudo-population is the one in Figure \ref{fig:time-dependent_confounding} (a) where $P(\vc{A}_t \mid \hist{X}_t, \hist{A}_{t-1}, V) = P(\vc{A}_t)$. On the other hand, when using stabilized weights $SW$, causal graph in the pseudo-population is the one in Figure \ref{fig:time-dependent_confounding} (b) where $P(\vc{A}_t \mid \hist{X}_t, \hist{A}_{t-1}, V) = P(\vc{A}_t \mid \hist{A}_{t-1})$.
	
	\textbf{Counterfactual Recurrent Networks}. Instead of using IPTW, we proposed building a representation of $\hist{X}_t, \hist{A}_{t-1}, V$ that is not predictive of treatment $\vc{A}_t$. At timestep $t$, we have $k$ different possible treatments $\vc{A}_t \in \{ A_{1}, \dots A_{K} \}$. We build a representation of the history and covariates and treatments that has the same distribution across the different possible treatments: $P( \Phi(\hist{X}_t, \hist{A}_{t-1}, \vc{V})\mid \vc{A}_t = A_{1}) = \dots = P( \Phi(\hist{X}_t, \hist{A}_{t-1}, \vc{V})\mid \vc{A}_t = A_{K})$. By breaking the association between past exposure and current treatments $\vc{A}_t$, we satisfy the causal graph in Figure \ref{fig:time-dependent_confounding} (a) and thus we remove the bias from time-dependent confounders. 
	
	\newpage
    \section{Proof of Theorem 1} \label{apx:proof_theorem_1}
    We first prove the following proposition.
    \begin{proposition}
    For fixed $\Phi$, let $x' = \Phi(\bar{\vc{h}}_t)$. Then the optimal prediction probabilities of $G_a$ are given by
    \begin{equation}
        {G^j_a}^*(x') = \frac{P^\Phi_j(x')}{\sum_{i=1}^K P^\Phi_i(x')}\,.
    \end{equation}
    \end{proposition}
    \begin{proof}
    For fixed $\Phi$, the optimal prediction probabilities are given by
    \begin{align}
        G_a^* = \arg \max_{G_a} \sum_{j = 1}^K \int_{x'}
        \log(G_a^j(x')) P^\Phi_j(x')dx' &&
	     \text{subject to } \sum_{j = 1}^K G_a^j(x') = 1 \,.
    \end{align}
    Maximising the value function pointwise and applying Lagrange multiplies, we get
    \begin{equation}
        G_a^* = \arg \max_{G_a} \sum_{j = 1}^K
        \log(G_a^j(x')) P^\Phi_j(x') + \lambda \Bigg(\sum_{j = 1}^K G_a^j(x') - 1\Bigg)\,.
    \end{equation}
    Setting the derivative (w.r.t. ${G_a^j}^*(x')$) to $0$ and solving for ${G_a^j}^*(x')$ we get
    \begin{equation}
        {G_a^j}^*(x') = - \frac{P^\Phi_j(x')}{\lambda}
    \end{equation}
    where $\lambda$ can now be solved for using the constraint to be $\lambda = - \sum_{i=1}^K P^\Phi_i(x')$. This gives the result.
    \end{proof}
    
    \begin{proof} (of \textbf{Theorem 1}) By substituting the expression from Proposition 1 into the minimax game defined in Eq. \ref{eq:minimax}, the objective for $\Phi$ becomes
    \begin{equation}
        \min_{\Phi} \sum_{j = 1}^K \mathbb{E}_{x' \sim P^\Phi_j} \Big[\log \bigg(\frac{P^\Phi_j(x')}{\sum_{i=1}^K P^\Phi_i(x')}\bigg)\Big]\,.
    \end{equation}
    We then note that
    \begin{align}
        \sum_{j = 1}^K \mathbb{E}_{x' \sim P^\Phi_j} \Big[\log \bigg(\frac{P^\Phi_j(x')}{\sum_{i=1}^K P^\Phi_i(x')}\bigg)\Big] + K \log K &= \sum_{j = 1}^K \bigg(\mathbb{E}_{x' \sim P^\Phi_j} \Big[\log \bigg(\frac{P^\Phi_j(x')}{\sum_{i=1}^K P^\Phi_i(x')}\bigg)\Big] + \log K \bigg) \\
        &= \sum_{j = 1}^K \mathbb{E}_{x' \sim P^\Phi_j} \Big[\log \bigg(\frac{P^\Phi_j(x')}{\frac{1}{K}\sum_{i=1}^K P^\Phi_i(x')}\bigg)\Big] \\
        &= \sum_{j=1}^K KL\bigg(P^\Phi_j(x') \bigg| \bigg| \frac{1}{K}\sum_{i=1}^K P^\Phi_i(x') \bigg) \\
        &= K \cdot JSD(P^\Phi_1, ..., P^\Phi_K)
    \end{align}
    where $KL(\cdot||\cdot)$ is the Kullback-Leibler divergence and $JSD(\cdot, ..., \cdot)$ is the multi-distribution Jensen-Shannon Divergence \citep{li2018deep}. Since $K \log K$ is a constant and the multi-distribution JSD is non-negative and $0$ if and only if all distributions are equal, we have that $P^\Phi_1 = ... = P^\Phi_K$.
    \end{proof}
 
    \newpage	
	\section{Training procedure for CRN} \label{apx:training_procedure}
	
	Let $\cl{D} = \left\{ \{\vc{x}^{(i)}_{t},  \vc{a}^{(i)}_{t}, \vc{y}_{t+1}^{(i)} \}_{t=1}^{T^{(i)}} \cup \{\vc{v}^{(i)} \}  \right\}_{i=1}^N$ be an observational dataset consisting of information about $N$ independent patients that we use to train CRN. The encoder and decoder networks part of CRN are trained into two separate steps.
	
	To begin with, the encoder is trained to built treatment invariant representations of the patient history and to perform one-step ahead prediction. After the encoder is optimized, we use it to compute the balancing representation $\vc{br}^{(i)}_t$ for each timestep in the trajectory of patient $(i)$. To train the decoder, we modify the training dataset as follows. For each patient $(i)$, we split their trajectory into shorter sequences of the $\tau_{\max}$ timesteps of the form: 
	\begin{equation}\left\{\mathbf{br}_l^{(i)} \cup \{\vc{y}^{(i)}_{l + t},  \vc{a}^{(i)}_{l + t}, \vc{y}_{l + t+1}^{(i)} \}_{t=1}^{\tau_{max}} \cup \vc{v}^{(i)}   \right\},
	\end{equation}
	for $l=1, \dots T^{(i)} - \tau_{\max}$. Thus, each patients contributes with $T^{(i)} - \tau_{\max}$ examples in the dataset for training the decoder. The different sequences obtained for all patents are randomly grouped into minibatches and used for training.  
	
	The pseudocode in Algorithm \ref{alg:pseudo} shows the training procedure used for the encoder and decoder networks part of CRN. The model was implemented in TensorFlow and trained on an NVIDIA Tesla K80 GPU. The Adam optimizer \citep{kingma2014adam} was used for training and both the encoder and the decoder are trained for 100 epochs. 
	
	\newpage
	
	\begin{algorithm}[h!]
		\caption{Pseudo-code for training CRN}\label{alg:pseudo}
		\begin{algorithmic} 
			\STATE \textbf{Input:} Training data: $\cl{D} = \left\{ \{\vc{x}^{(i)}_{t},  \vc{a}^{(i)}_{t}, \vc{y}_{t+1}^{(i)} \}_{t=1}^{T^{(i)}} \cup \vc{v}^{(i)}  \right\}_{i=1}^N$
			\STATE
			\STATE \textbf{(1) Encoder optimization}: parameters $\theta_{E, r},
			 \theta_{E, a}, \theta_{E, y}$.
			\STATE Learning rate: $\mu$
			\FOR{$p = 1, \dots, \text{max epochs}$}
			\STATE $ \lambda_p = \dfrac{2}{1 + \exp(-10\cdot p)} - 1$
			\FOR{Batch $\cl{B} = \left\{ \{\vc{x}^{(i)}_{t},  \vc{a}^{(i)}_{t}, \vc{y}_{t+1}^{(i)} \}_{t=0}^{T^{(i)}} \cup \vc{v}^{(i)}   \right\}_{i=1}^{|\mathcal{B}|}$ in epoch}
			\STATE Compute $\mathcal{L}_{E,a}^{\mathcal{B}}( \theta_{E, r}, \theta_{E, a}) = \frac{1}{|\mathcal{B}|} \sum_{i \in \mathcal{B}} \sum_{t=1}^{T^{(i)}} \mathcal{L}_{t, a}^{(i)}( \theta_{E, r}, \theta_{E, a})$
			
			\STATE Compute $\mathcal{L}_{E,y}^{\mathcal{B}}( \theta_{E, r}, \theta_{E, y}) = \frac{1}{|\mathcal{B}|} \sum_{i \in \mathcal{B}} \sum_{t=1}^{T^{(i)}} \mathcal{L}_{t, y}^{(i)}( \theta_{E, r}, \theta_{E, y})$
			
			\STATE $\theta_{E, r} \leftarrow \theta_{E, r} - \mu \left (\frac{\partial \mathcal{L}_{E, y}^{\mathcal{B}}( \theta_{E, r}, \theta_{E, y})}{\partial \theta_{E, r}} -\lambda_p \frac{\partial \mathcal{L}_{E, a}^{\mathcal{B}}( \theta_{E, r}, \theta_{E, a})}{\partial \theta_{E, r}} \right)$
			\STATE $\theta_{E, y} \leftarrow \theta_{E, y} - \mu \frac{\partial \mathcal{L}_{E, y}^{\mathcal{B}}( \theta_{E, r}, \theta_{E, y})}{\partial\theta_{E, y}}$
			\STATE $\theta_{E, a} \leftarrow \theta_{E, a} - \mu \frac{\partial \mathcal{L}_{E, a}^{\mathcal{B}}( \theta_{E, r}, \theta_{E, a})}{\partial\theta_{E, a}}$
			\ENDFOR
			\ENDFOR
			\STATE
			\STATE \textbf{(2) Compute the encoder balanced representation and use it to initialize the decoder hidden state.}
			\FOR {$i = 1, ..., N$}
			\FOR {$t = 1, \dots, T^{(i)}$}
			\STATE $\mathbf{br}_t^{(i)} =  \text{encoder}(\hist{x}^{(i)}_t, \hist{a}^{(i)}_{t-1}, \vc{v}^{(i)}; \theta_{E, r})$
			\ENDFOR
			\ENDFOR
			
			\STATE
			\STATE \textbf{(3) Split dataset in sequences of $\tau_{\max}$ timesteps}: 
			$$\left\{ \left\{ \mathbf{br}_l^{(i)} \cup \{\vc{y}^{(i)}_{l + t},  \vc{a}^{(i)}_{l + t}, \vc{y}_{l + t+1}^{(i)} \}_{t=1}^{\tau_{max}} \cup \vc{v}^{(i)}   \right\}_{l=1}^{T^{(i)} - \tau_{\max}}\right\}_{i=1}^{N}$$
			
			\STATE
			\STATE \textbf{(4) Optimize decoder}: parameters $\theta_{D, r}, \theta_{D, a}, \theta_{D, y}$
			\STATE Learning rate: $\mu$
			\FOR{p = 1, \dots, max epochs}
			\STATE $ \lambda_p = \dfrac{2}{1 + \exp(-10\cdot p)} - 1$
			\FOR{Batch $\cl{B} = \left\{ \mathbf{br}_l^{(i)} \cup \{\vc{y}^{(i)}_{l + t},  \vc{a}^{(i)}_{l + t}, \vc{y}_{l + t+1}^{(i)} \}_{t=0}^{\tau_{max}} \cup \{\vc{v}^{(i)} \}  \right\}_{i=1}^{|\mathcal{B}|}$ in epoch}
			\STATE Compute $\mathcal{L}_{D,a}^{\mathcal{B}}( \theta_{D, r}, \theta_{D, a}) = \frac{1}{|\mathcal{B}|} \sum_{i \in \mathcal{B}} \sum_{t=1}^{\tau_{max}} \mathcal{L}_{t, a}^{(i)}( \theta_{D, r}, \theta_{D, a})$
			
			\STATE Compute $\mathcal{L}_{D,y}^{\mathcal{B}}( \theta_{D, r}, \theta_{D, y}) = \frac{1}{|\mathcal{B}|} \sum_{i \in \mathcal{B}} \sum_{t=1}^{\tau_{max}} \mathcal{L}_{t, y}^{(i)}( \theta_{D, r}, \theta_{D, y})$
			
			\STATE $\theta_{D, r} \leftarrow \theta_{D, r} - \mu \left (\frac{\partial \mathcal{L}_{D, y}^{\mathcal{B}}( \theta_{D, r}, \theta_{D, y})}{\partial \theta_{D, r}} -\lambda_p \frac{\partial \mathcal{L}_{D, a}^{\mathcal{B}}( \theta_{D, r}, \theta_{D, a})}{\partial \theta_{D, r}} \right)$
			\STATE $\theta_{D, y} \leftarrow \theta_{D, y} - \mu \frac{\partial \mathcal{L}_{D, y}^{\mathcal{B}}( \theta_{D, r}, \theta_{D, y})}{\partial\theta_{D, y}}$
			\STATE $\theta_{D, a} \leftarrow \theta_{D, a} - \mu \frac{\partial \mathcal{L}_{D, a}^{\mathcal{B}}( \theta_{D, r}, \theta_{D, a})}{\partial\theta_{D, a}}$
			\ENDFOR
			\ENDFOR
			\STATE
			\STATE \textbf{Output: }Trained CRN encoder (parameters $\theta_{E, r}, \theta_{E, a}, \theta_{E, y}$) and trained CRN decoder (parameters $\theta_{D, r}, \theta_{D, a}, \theta_{D, y}$. )
		\end{algorithmic} 
	\end{algorithm}

	\newpage
	
	\section{Pharmacokinetic-Pharmacodynamic Model of tumour growth}
	\label{apx:model_of_tumour_growth}
	
	To evaluate the CRN on counterfactual estimation, we need access to the data generation mechanism to build a test set that consists of patient outcomes under all possible treatment options. For this purpose, we use the state-of-the-art pharmacokinetic-pharmacodynamic (PK-PD) model of tumour growth proposed by \cite{geng2017prediction} and also used by \cite{lim2018forecasting} for evaluating RMSMs. The PK-PD model characterizes patients suffering from non-small cell lung cancer and models the evolution of their tumour under the combined effects of chemotherapy and radiotherapy. In addition, the model includes different distributions of tumour sizes based on the cancer stage at diagnosis.
	
	\textbf{Model of tumour growth} The volume of tumour $t$ days after diagnosis is modelled as follows:
	\begin{equation}
	\begin{aligned}
	V(t+1) =  \Big( 1 + \underbrace{\rho \text{log}(\dfrac{K}{V(t)})}_{\text{Tumor growth}} - \underbrace{\beta_c C(t)}_{\text{Chemotherapy}} - \underbrace{\left(\alpha_r d(t) + \beta_r d(t)^2\right)}_{\text{Radiotherapy}} + \underbrace{e_t}_{\text{Noise}}  \Big)V(t) 
	\end{aligned}
	\end{equation}
	where the parameters $K, \rho, \beta_c, \alpha_r, \beta_r$ are sampled from the prior distributions described in \citep{geng2017prediction} and $e_t \sim \mathcal{N}(0, 0.01^2)$ is a noise term that accounts for randomness in the tumour growth. 
	
	To incorporate heterogeneity among patient responses, due to, for instance, gender or genetic factors \cite{bartsch2007genetic}, the prior means for $\beta_c$ and  $\alpha_r$ are adjusted to create three patient subgroups $S^{(i)} \in \{1, 2, 3\}$ as described in \cite{lim2018forecasting}. This way, we incorporate in the model of tumour growth specific characteristics that affect the patient's individualized response to treatments. Thus, the prior mean $\mu_{\beta_c}$ of $\beta_c$ and the prior mean $\mu_{\alpha_r}$ of $\alpha_r$ are augmented as follows.
	\begin{align}
	    \mu^{\prime}_{\beta_c}(i) &= \begin{cases}
               1.1 \mu_{\beta_c}, \text{if } S^{(i)} = 3 \\
              \mu_{\beta_c}, \text{otherwise}
            \end{cases} & 
         \mu^{\prime}_{\alpha_r}(i) &= \begin{cases}
               1.1 \mu_{\alpha_r}, \text{if} S^{(i)} = 1 \\
              \mu_{\alpha_r}, \text{otherwise}
            \end{cases}
	\end{align}
	where $\mu_{\beta_c}$ and $\mu_{\alpha_r}$ are the mean parameters from \cite{geng2017prediction} and $\mu^{\prime}_{\beta_c}(i)$ and $\mu^{\prime}_{\alpha_r}(i)$ are the parameters used in the data simulation. The patient subgroup $S^{(i)} \in \{1, 2, 3\}$ is used as baseline features.
	
	 The chemotherapy drug concentration follows an exponential decay with half life of 1 day:
	 \begin{equation}
	     C(t) = \tilde{C}(t) + C(t-1)/2,
	 \end{equation}
	 where $ \tilde{C}(t) = 5.0 mg/m^3$ of Vinblastine if chemotherapy is given at time $t$. $d(t) = 2.0 Gy$ fractions of radiotherapy if the radiotherapy treatment is applied at timestep $t$. 
	
	Time-varying confounding is introduced by modelling  chemotherapy and radiotherapy assignment as Bernoulli random variables, with probabilities $p_c$ and $p_r$ depending on the tumour diameter: 
	\begin{align}
	p_c(t) = \sigma \left(\frac{\gamma_c}{D_{\max}} (\bar{D}(t) - \delta_c )  \right) && p_r(t) = \sigma \left(\frac{\gamma_r}{D_{\max}} (\bar{D}(t) - \delta_r ) \right),
	\end{align}
	where $\bar{D}(t)$ is the average tumour diameter over the last $15$ days, $D_{\max} = 13 \text{cm}$ is the maximum tumour diameter and $\sigma(\cdot)$ is the sigmoid activation function. The parameters $\delta_c$ and $\delta_r$ are set to $\delta_c = \delta_r = D_{\max}/2$ such that there is $0.5$ probability of receiving treatment when tumour is half of its maximum size. $\gamma_c, \gamma_r$ control the amount of time-dependent confounding; the higher $\gamma_{\star}$ is, the more important the history of tumour diameter is in assigning treatments. Thus, at each timestep, there are four treatment options options: no treatment ($A_1$),  chemotherpy ($A_2$),  radiotherapy ($A_3$), combined chemotherapy and radiotherapy ($A_4$). 

    Since the work most relevant to ours is the one of \cite{lim2018forecasting} we used the same data simulation and same settings for $\gamma = \gamma_c = \gamma_r$ as in their case. When $\gamma = 0$, there is no time-dependent confounding and the treatments are randomly assigned.  By increasing $\gamma$ we increase the influence of the volume size history (encoded in $\bar{D}(t)$) on the treatment probability. For example, assume $\bar{D}(t) = \frac{3D_{max}}{4}$. From equation (7), the probability of chemotherapy in this case is $p_c (t) = \sigma(\frac{\gamma_c}{D_{max}} (\bar{D}(t) - \frac{D_{max}}{2})) = \sigma(0.25\gamma_c)$, where $\sigma(\cdot)$ is the sigmoid function. When $\gamma = 1$, $p_c (t) = 0.56$, when $\gamma = 5$, $p_c (t) = 0.77$ and when $\gamma = 10$, $p_c (t) = 0.92$ in this example. $\gamma$ can be increased further to increase the bias. However, the values used in the experiments evaluate the model on a wide range of settings for the time-dependent confounding bias. 
    
    \newpage

	\section{Marginal Structural Models} \label{apx:marginal_structural_models}
	
	Marginal Structural Models \citep{robins2000marginal, hernan2001marginal} have been widely used in epidemiology and as part of follow up studies. In our case, we would like to estimate the effects of a sequence of treatments in the future given the current patient history: 
	\begin{equation}
	\mathbb{E}(\vc{Y}_{t+\tau} \mid \hist{A}(t, t+\tau -1 ) =  \hist{a}(t, t+\tau -1 ),  \bar{\vc{H}}_{t}) = g(\tau, a(t, t+\tau-1), \hist{H}_t) ,
	\end{equation} 
	where $g$ is a generic function and $\hist{a}(t, t+\tau-1) = [\vc{a}_{t}, \dots \vc{a}_{{t+\tau -1 }}]$ represents a possible sequence of treatments from timestep $t$ just until before the potential outcome $\vc{Y}_{t+\tau}$ is observed. After removing the bias form time-dependent confounders, $\mathbb{E}(\vc{Y}_{t+\tau} \mid \hist{A}(t, t+\tau -1 ) =  \hist{a}(t, t+\tau -1 ),  \bar{\vc{H}}_{t}) = \mathbb{E}(\vc{Y}_{t+\tau}[\hist{a}(t, t+\tau -1 )]$.
	
	Note that for implementing MSMs, we encode the treatments at timestep $t$ in the model of tumour growth as $\vc{A}_t = [A_{t, c}, A_{t, d}]$ to indicate the binary application of chemotherapy and radiotherapy. In order to remove the time-dependent confounding bias and estimate future outcomes, we use the stabilized weights of MSMs to weight each patient in the dataset: 
	\begin{equation}
	SW(t, \tau) = \prod_{n=t}^{t+\tau} \dfrac{f(\vc{A}_n \mid\hist{A}_{n-1})}{f(\vc{A}_n \mid \hist{A}_{n-1}, \hist{X}_{n}, \vc{V})} =  \prod_{n=t}^{t+\tau} \dfrac{\prod_{k\in\{c, d\}} f(A_{n, k} \mid\hist{A}_{n-1})}{ \prod_{k\in\{c, d\}} f(A_{n, k} \mid \hist{A}_{n-1}, \hist{X}_{n}, \vc{V})},
	\end{equation}
	where $f(\cdot)$ represents the conditional probability mass function for discrete treatments.

	We adopt the implementation in \citep{hernan2001marginal, howe2012estimating, lim2018forecasting} for MSMs and use logistic regression for estimating the propensity weights as follows: 
	\begin{equation}
	f(A_{t,k} \mid \hist{A}_{t-1}) = \sigma \Big(\sum_{j=1}^k \omega_{k}  (\sum_{i=1}^{t-1} A_{t,j}) \Big)
	\end{equation}
	\begin{equation}
	f(A_{t,k} \mid  \hist{H}_{t}) = \sigma \Big(\sum_{k\in\{c, d\}} \phi_{k}  (\sum_{i=1}^{t-1} A_{t,k})   
	+ \vc{w}_1 \vc{X}_t + \vc{w}_2 \vc{X}_{t-1} + \vc{w}_3 \vc{V} \Big)
	\end{equation}
	where $\omega_{\star}, \phi_{\star}$ and $\vc{w}_{\star}$ are regression coefficients, $k\in\{c, d\}$ indicates the chemotherapy or radiotherapy treatments and $\sigma(\cdot)$ is the sigmoid function. 
	
	For predicting the outcome, the following regression model is used, where each individual patient is weighted by its propensity score:
	\begin{equation}
	g(\tau, a(t, t+\tau-1),  \hist{H}_t) = \sum_{k\in\{c, d\}} \beta_{k}  (\sum_{n=t}^{t+\tau-1} A_{n,k}) + \vc{l}_1 \vc{X}_t  
	+ \vc{l}_2 \vc{X}_{t-1} +  \vc{l}_3 \vc{V} 
	\end{equation}
	where $\beta_{\star}$ and $\vc{l}_{\star}$ are regression coefficients.
	
	MSMs do not require hyperparameter tuning so we use the patients from both the train and validation sets for training. 
		
	\section {Recurrent Marginal Structural Networks} \label{apx:recurrent_marginal_structural_networks}
	
	MSMs are very sensitive to model mis-specification in computing the propensity weights and estimating the outcomes. Recurrent Marginal Structural Models (RMSNs) \citep{lim2018forecasting} overcome this problem by using recurrent neural networks to estimate the propensity scores and to build the outcome model. RNNs are more robust to changes in the treatment assignment policy. 
	RMSNs were implemented as descried in \cite{lim2018forecasting}\footnote{We used the publicly available implementation from \url{https://github.com/sjblim/rmsn_nips_2018}.}. 
	
	For implementing RMSNs, we also encode the treatments at timestep $t$ in the model of tumour growth as $\vc{A}_t = [A_{t, c}, A_{t, d}]$ to indicate the binary application of chemotherapy and radiotherapy. The propensity weights are estimated using recurrent neural networks as follows:
	\begin{align}
	f(A_{t,k} \mid \hist{A}_{t-1}) = \text{RNN}_{SW_n}(\hist{A}_{t-1}) && f(A_{t,k} \mid  \hist{X}_{t}, \hist{A}_{t-1}) = \text{RNN}_{SW_d}(\hist{A}_{t-1}, \hist{X}_{t}, \vc{V})
	\end{align}
	
	For predicting one-step-ahed outcome, R-MSNs use an encoder network:
	\begin{eqnarray}
	g(1, a(t, t), \hist{H}_t) =  \text{RNN}_{E}(\vc{a}_t, \hist{A}_{t-1}, \hist{X}_{t}, \vc{V}),
	\end{eqnarray}
	where in the loss function, each patient is weighted by their stabilized IPTW. 
	
	For estimating the treatment responses for a sequence of treatments in the future, RMSNs use a decoder network:
	\begin{eqnarray}
	g(\tau, a(t, t + \tau - 1), \hist{H}_t) =  \text{RNN}_{D}(\vc{a}_t, \dots, \vc{a}_{t+\tau-1}, \hist{A}_{t-1}, \hist{X}_{t}, \vc{V}). 
	\end{eqnarray}
	See \cite{lim2018forecasting} for more details about the R-MSNs model architecture and training procedure of the propensity weights, encoder and decoder networks. Tables \ref{tab:hyperparameters_rmsn_propensity_encodeer} and \ref{tab:hyperparameters_rmsn_decoder} show the hyperparameter search ranges used to optimize this model for evaluation in our paper. The hyperparameters were selected in the same way as proposed by \cite{lim2018forecasting}, based on the error on the factual outcomes in the validation dataset. All of the models are trained using Adam optimizer for 100 epochs.  
	
	\begin{table}[ht]
		 	\caption{Hyperparameter search range for propensity networks and encoder (same as in \cite{lim2018forecasting}). C is the size of the input. }
		 	\label{tab:hyperparameters_rmsn_propensity_encodeer}
		 	\vskip 0.15in
		 	\begin{center}
		 		\begin{small}
		 			\begin{tabular}{cc}
		 				\toprule
		 				Hyperparameter & Search range \\
		 				\hline
		 				\hline
		 				Iterations of Hyperparameter Search & 50 \\
		 				Learning rate &    0.01, 0.005, 0.001  \\
		 				Minibatch size &  64, 128, 256  \\
		 				RNN state size &  0.5C, 1C, 2C, 3C, 4C   \\
		 				Dropout rate &  0.1, 0.2, 0.3, 0.4, 0.5 \\
		 				Max Gradient Norm & 0.5, 1.0, 2.0 \\
		 				\bottomrule
		 			\end{tabular}
		 		\end{small}
		 	\end{center}
		 	\vskip -0.1in
		 \end{table}
 
	 	\begin{table}[ht]
	 	\caption{Hyperparameter search range for decoder (same as in \cite{lim2018forecasting}). C is the input size. }
	 	\label{tab:hyperparameters_rmsn_decoder}
	 	\vskip 0.15in
	 	\begin{center}
	 		\begin{small}
	 			\begin{tabular}{cc}
	 				\toprule
	 				Hyperparameter & Search range \\
	 				\hline
	 				\hline
	 				Iterations of Hyperparameter Search & 20 \\
	 				Learning rate &    0.01, 0.001, 0.0001  \\
	 				Minibatch size &  256, 512, 1024  \\
	 				RNN state size &   1C, 2C, 4C, 8C, 16C   \\
	 				Dropout Rate &  0.1, 0.2, 0.3, 0.4, 0.5 \\
	 				Max Gradient Norm & 0.5, 1.0, 2.0, 4.0 \\
	 				\bottomrule
	 			\end{tabular}
	 		\end{small}
	 	\end{center}
	 	\vskip -0.1in
	 \end{table}

	\section{Baseline RNN and linear model} \label{apx:baseline_RNN}
	
	For the baseline linear model, we fit the same regression model used for Marginal Structural Networks, but without using the IPTW. The baseline RNN uses an LSTM unit and, at each timestep, receives as input the current treatment, the patient covariates and the patient static features to perform one-step-ahead prediction. To have a model of similar capacity to the CRN (similar number of parameters), we add a fully connected layer on top of the output of the LSTM unit in order to obtain the outcomes. Table \ref{tab:hyperparameters_baseline_rnn} shows the hyperparameter search range used to optimize this model.
	The hyperparameters were selecting according to the error on the factual outcomes in the validation set. We train the baseline RNN using the Adam optimizer for 100 epochs. 
	
	\begin{table}[ht]
		\caption{Hyperparameter search range for baseline RNN model. C is the size of the input.}
		\label{tab:hyperparameters_baseline_rnn}
		\vskip 0.15in
		\begin{center}
			\begin{small}
				\begin{tabular}{cc}
					\toprule
					Hyperparameter & Search range \\
					\hline
					\hline
					Iterations of Hyperparameter Search & 50 \\
					Learning rate &    0.01, 0.001, 0.0001  \\
					Minibatch size &  64, 128, 256  \\
					RNN hidden units &  0.5C, 1C, 2C, 3C, 4C   \\
					FC hidden units & 0.5C, 1C, 2C, 3C, 4C \\
					RNN dropout probability &  0.1, 0.2, 0.3, 0.4, 0.5 \\
					\bottomrule
				\end{tabular}
			\end{small}
		\end{center}
		\vskip -0.1in
	\end{table}
	
	\newpage
	\section{Hyperparameter optimization for CRN} \label{apx:hyperparameter_optimization}
	
	As described in Appendix C, the dataset for training the decoder are used by splitting the sequences of the patients in the training set. This creates a larger dataset for training (where each patient $(i)$ contributes $T^{(i)} - \tau_{\max}$ times to the dataset) which requires a different hyperparameter search range. Moreover, the balancing representations computed by the encoder are used to initialize the state of the RNN for the decoder. Thus, the decoder RNN size is equal to the size of the balancing representation size of the encoder. Table \ref{tab:hyperparameters_crn_encoder} shows the hyperparameter search ranges for the encoder and decoder networks in CRN. We selected hyperparameters based on the error of the model on the factual outcomes in the validation dataset. All models are trained for 100 epochs.

	In addition, Tables \ref{tab:optimal_hyperparameters_CRN_encoder} and \ref{tab:optimal_hyperparameters_CRN_decoder} illustrate the optimal hyperparameters chosen. 
	
	\begin{table}[ht]
		\caption{Hyperparameter search range for CRN encoder. C is the size of the input and R is the size of the balancing representation.}
		\label{tab:hyperparameters_crn_encoder}
		\vskip 0.15in
		\begin{center}
			\begin{small}
				\begin{tabular}{ccc}
					\toprule
					Hyperparameter & Search range encoder & Search range decoder \\
					\hline 
					\hline
					Iterations of Hyperparameter Search & 50 & 30\\
					Learning rate &    0.01, 0.001, 0.0001  & 0.01, 0.001, 0.0001  \\
					Minibatch size &  64, 128, 256  & 256, 512, 1024 \\
					RNN hidden units &  0.5C, 1C, 2C, 3C, 4C & Balancing representation size of encoder \\
					Balancing representation size &  0.5C, 1C, 2C, 3C, 4C &  0.5C, 1C, 2C, 3C, 4C \\
					FC hidden units & 0.5R, 1R, 2R, 3R, 4R & 0.5R, 1R, 2R, 3R, 4R\\
					RNN dropout probability &  0.1, 0.2, 0.3, 0.4, 0.5 &  0.1, 0.2, 0.3, 0.4, 0.5 \\
					\bottomrule
				\end{tabular}
			\end{small}
		\end{center}
		\vskip -0.1in
	\end{table}

	\begin{table*}[ht]
		\caption{Optimal hyperparameters for the CRN encoder when different degrees of time-dependent confounding are applied in the model of tumour growth. The parameters $\gamma_c$ and $\gamma_r$ measure the degree of time-dependent confounding applied. When $\gamma_c$ and $\gamma_r$ are set to the same value, we denote this with $\gamma_{\star}$. }
		\label{tab:optimal_hyperparameters_CRN_encoder}
		\vskip 0.15in
		\begin{center}
			\begin{small}
				\begin{tabular}{cccccccc}
					\toprule
					 & $\gamma_{\star}=0$ & $\gamma_{\star}=1$ & $\gamma_{\star}=2$  & $\gamma_{\star}=3$ & $\gamma_{\star}=4$ & $\gamma_{\star}=5$ &  \\
					\hline
					\hline
					Learning rate &    0.001  & 0.1 & 0.001 & 0.01 & 0.01 & 0.001 \\
					Minibatch size &  64 & 64 & 64 & 128 & 64 & 128\\
					RNN hidden units & 12 & 18 & 24 & 18 & 24 & 24  \\
					Balancing representation size  & 18 & 18 & 12 & 18 & 6 & 12  \\
					FC hidden units & 18 & 18 & 36 & 54 & 24 & 48  \\
					RNN dropout probability &  0.1 & 0.1 & 0.1 & 0.2   & 0.2 &  0.1\\
					\midrule 
					\midrule 
					 &$\gamma_{\star}=6$ & $\gamma_{\star}=7$ & $\gamma_{\star}=8$  & $\gamma_{\star}=9$ & $\gamma_{\star}=10$ \\
					\hline
					\hline
					Learning rate  & 0.001 & 0.001 & 0.01 & 0.001 & 0.01 &   \\
					Minibatch size  & 64 & 64 & 128 & 128 & 128 & \\
					RNN hidden units & 24 & 18 &  12& 24& 24 &  \\
					Balancing representation size & 12  & 18 & 24 & 18 & 12 & \\
					FC hidden units & 48 & 72 & 12 & 36 & 12 \\
					RNN dropout probability & 0.1 & 0.2 & 0.1 & 0.1 & 0.1 & \\
					\midrule 
					\midrule 
					& \multicolumn{2}{l}{$\gamma_c = 0$, $\gamma_r=5$} & $\gamma_c = 5$, $\gamma_r=0$ \\
					\hline
					\hline
					Learning rate  & 0.01 & & 0.001  \\
					Minibatch size  & 128 & &64\\
					RNN hidden units & 12 & & 12 \\
					Balancing representation size  & 18 & &24\\
					FC hidden units  & 36 & & 96\\
					RNN dropout probability & 0.1 & &0.1\\
					\bottomrule
				\end{tabular}
			\end{small}
		\end{center}
		\vskip -0.1in
	\end{table*}

	\begin{table*}[t]
	\caption{Optimal hyperparameters for the CRN decoder when different degrees of time-dependent confounding are applied in the model of tumour growth. The parameters $\gamma_c$ and $\gamma_r$ measure the degree of time-dependent confounding applied. When $\gamma_c$ and $\gamma_r$ are set to the same value, we denote this with $\gamma_{\star}$ }
	\label{tab:optimal_hyperparameters_CRN_decoder}
	\vskip -0.1in
	\begin{center}
	\begin{small}
		\begin{tabular}{cccccc}
			\toprule
			& $\gamma_{\star}=1$ & $\gamma_{\star}=2$ & $\gamma_{\star}=3$  & $\gamma_{\star}=4$ & $\gamma_{\star}=5$    \\
			\hline
			\hline
			Learning rate &    0.001  & 0.001 & 0.001 & 0.001 & 0.001   \\
			Minibatch size &  1024 & 1024 & 512 & 1024 & 1024 \\
			RNN hidden units & 18 & 12 & 18 & 6 & 12  \\
			Balancing representation size  & 18 & 18 & 6 & 18 & 3  \\
			FC hidden units & 18 & 36 & 18 & 72 & 6  \\
			RNN dropout probability &  0.1 & 0.2 & 0.3 & 0.1 & 0.1 \\
			\midrule 
			\midrule
			& $\gamma_c = 0$, $\gamma_r=5$ & $\gamma_c = 5$, $\gamma_r=0$  \\
			\hline
			\hline
			Learning rate  & 0.01 &0.001  \\
			Minibatch size & 512 & 1024 \\
			RNN hidden units &  18& 24 \\
			Balancing representation size  & 18 & 12 \\
			FC hidden units &36& 24 \\
			RNN dropout probability & 0.1& 0.03\\
			\bottomrule
		\end{tabular}
	\end{small}
	\end{center}
	\vskip -0.1in
	\end{table*}

    \clearpage

	\section{Full results for counterfactual prediction} \label{apx:full_results_counterfactuals}
	
		\subsection{Multi-step ahead prediction of counterfactuals}
		
		Figure \ref{fig:multi_step_ahead} shows the normalized RMSE for multiple step-ahead prediction of counterfactuals. The RMSE is normalized by the maximum tumour volume: $V_{max} = 1150 \text{cm}^3$. The counterfactuals in this case are generated as described in Section 6.3 and Appendix I. We notice that performance gains of CRN compared to RMSN increase with the number of future timesteps for which the counterfactuals are estimated.  
		
			\begin{figure}[h]
				\centering
				\subfloat[\textbf{Two}-step ahead prediction]{{\includegraphics[height=5.5cm]{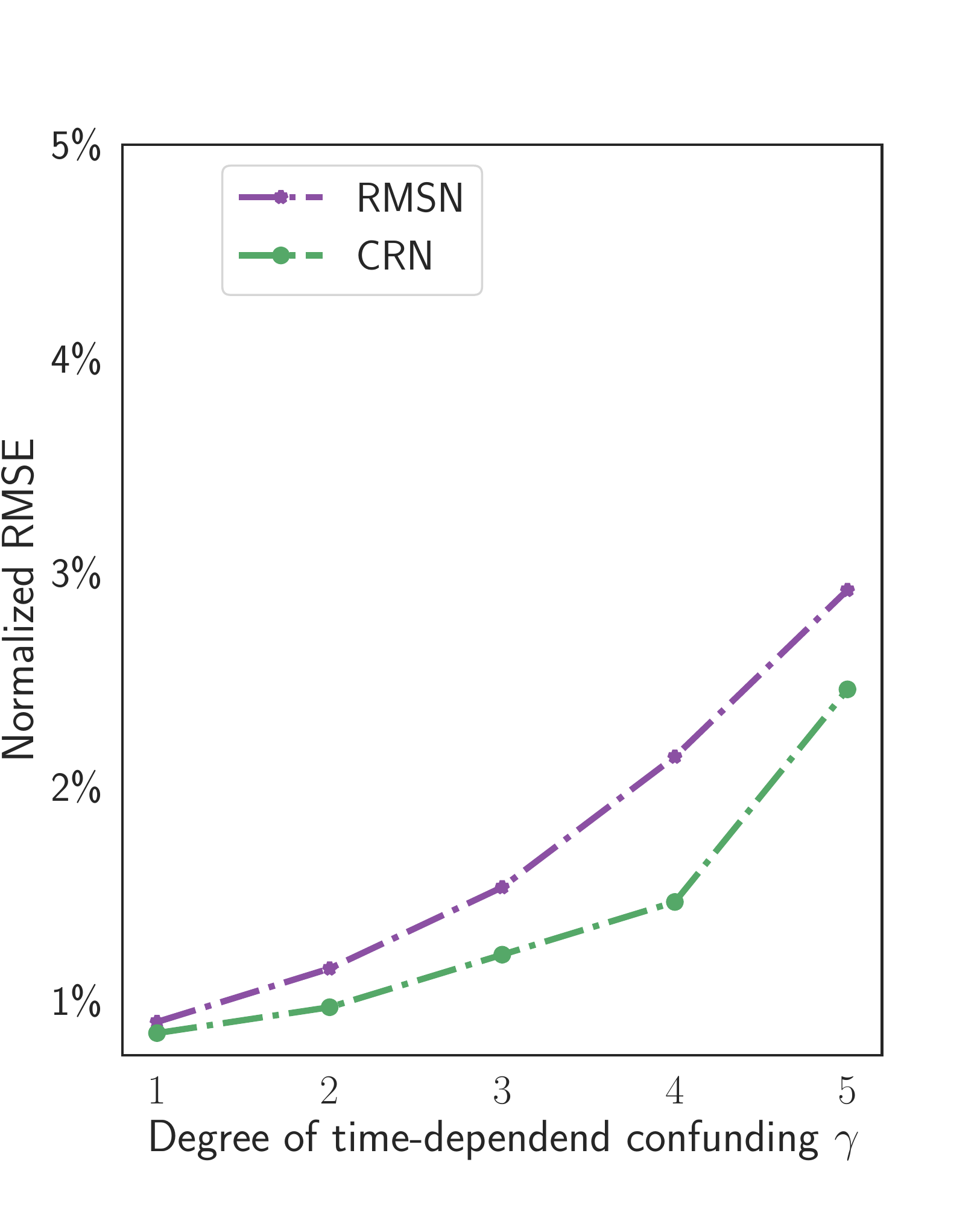} }}%
				\subfloat[\textbf{Three}-step ahead prediction]{{\includegraphics[height=5.5cm]{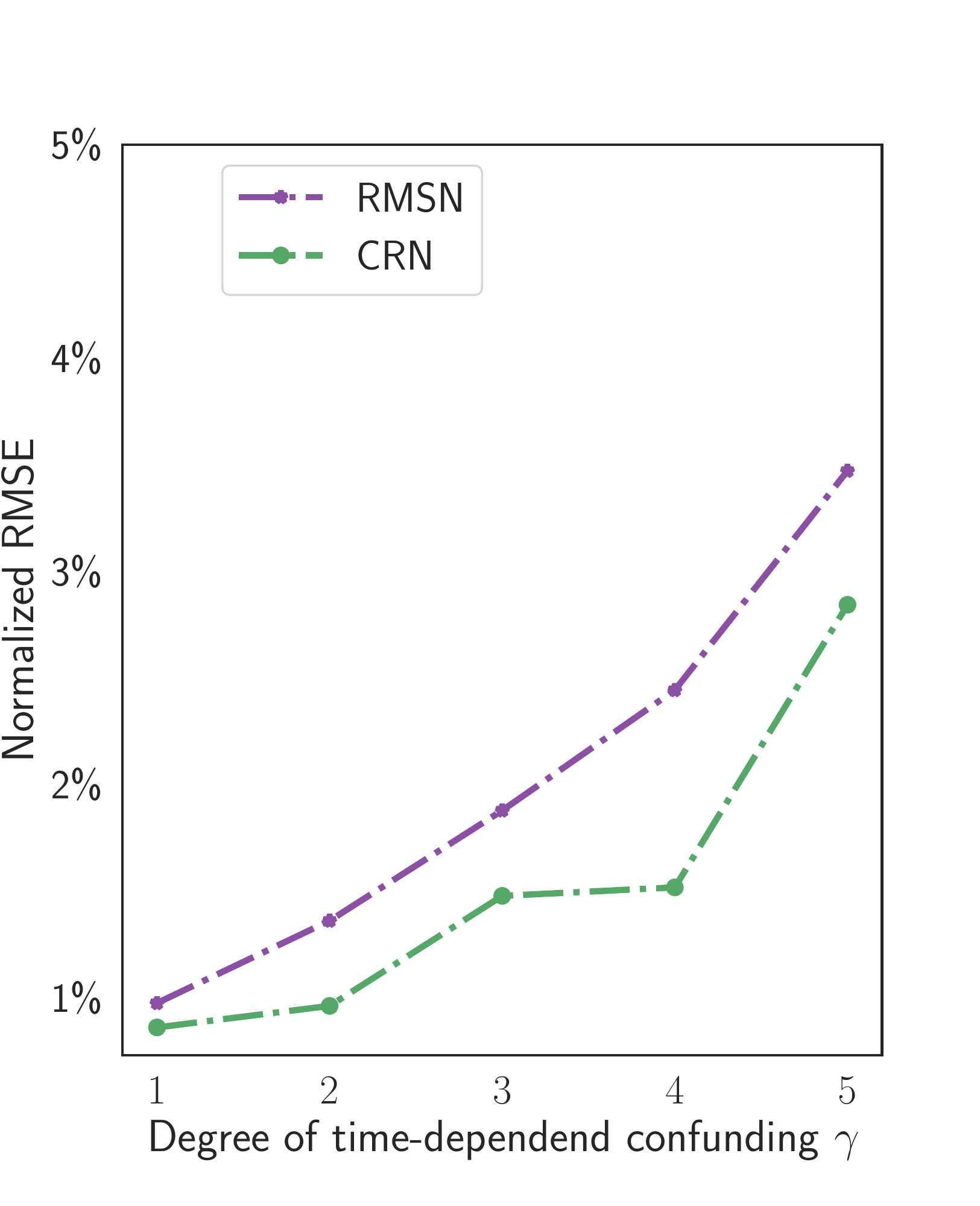} }}%
				\subfloat[\textbf{Four}-step ahead prediction]{{\includegraphics[height=5.5cm]{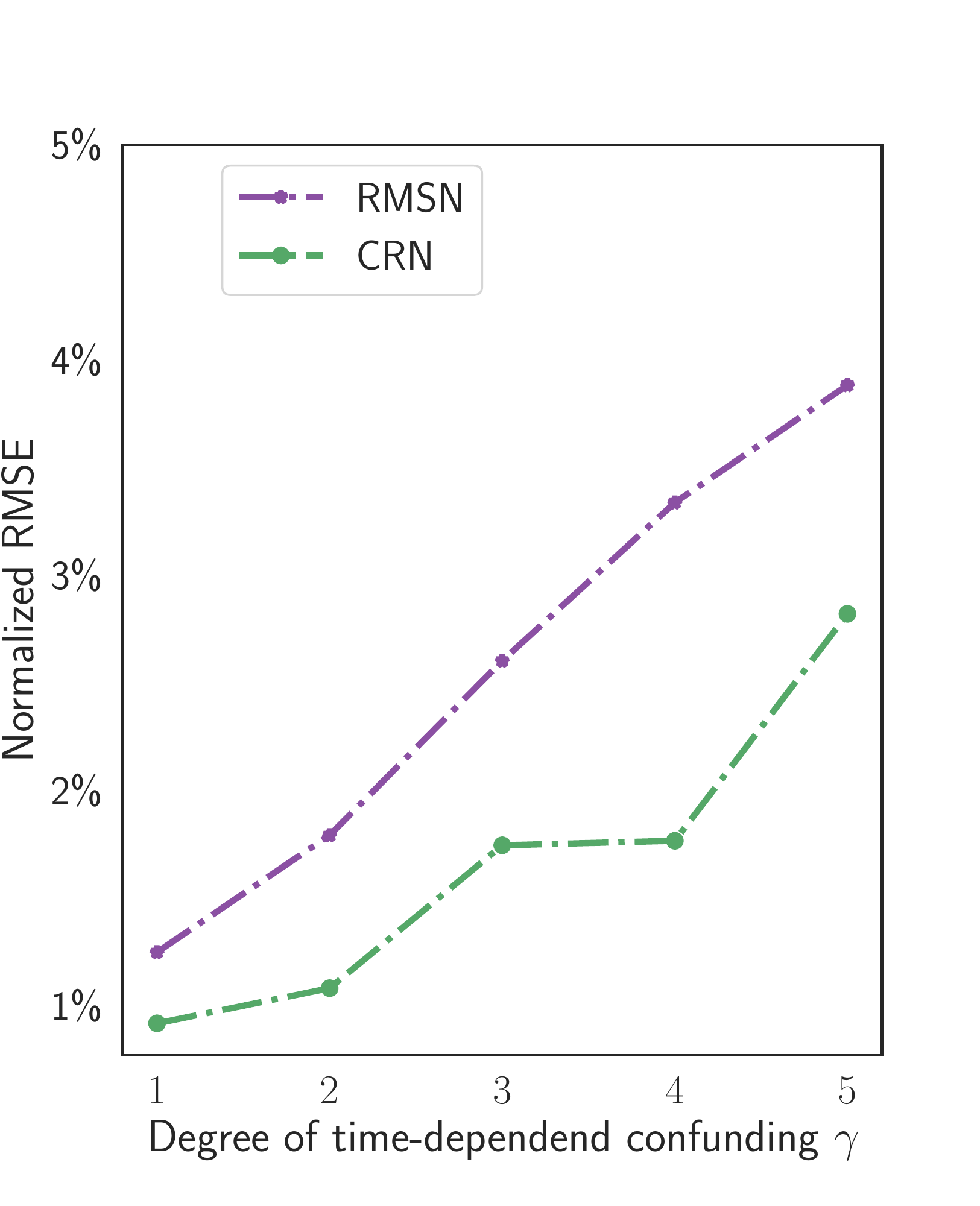} }}%
				
				\subfloat[\textbf{Five}-step ahead prediction]{{\includegraphics[height=5.5cm]{figs/seq_figures/results_5_sequence_prediction} }}%
				\subfloat[\textbf{Six}-step ahead prediction]{{\includegraphics[height=5.5cm]{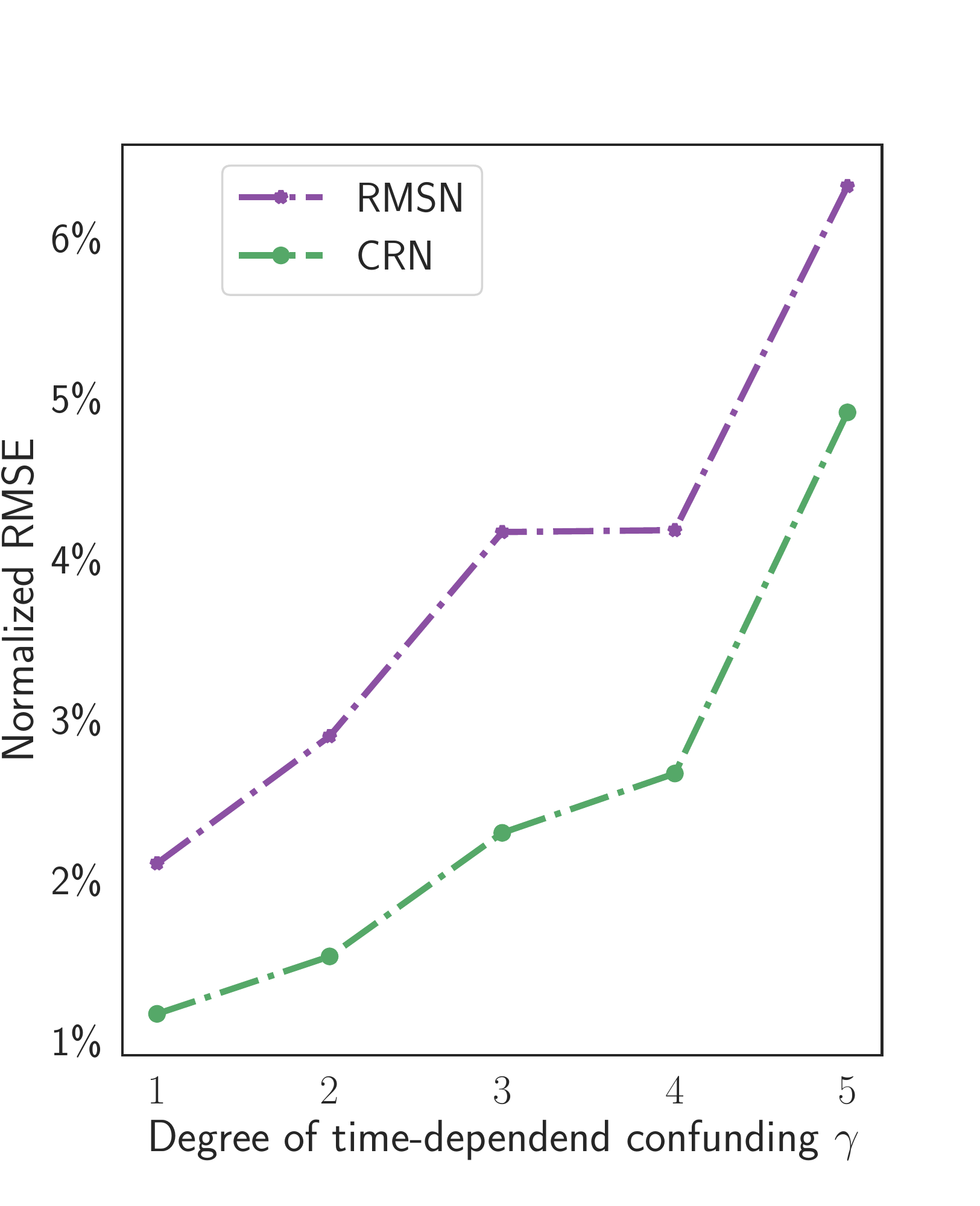} }}%
				\caption{Results for prediction of patient counterfactuals for multiple steps ahead.}%
				\label{fig:multi_step_ahead}
			\end{figure}

		\newpage
		\subsection{Detailed results for the counterfactual predictions}
		Tables \ref{tab:results_one_step_ahead_prediction} and \ref{tab:results_sequence_prediction} show detailed results for the counterfactual predictions. 
	
		\begin{table*}[ht]
		\caption{Normalized RMSE for one-step-ahead prediction of counterfactuals. The parameter $\gamma$ measures the degree of time-dependent confounding applied.  }
		\label{tab:results_one_step_ahead_prediction}
		\vskip 0.15in
		\begin{center}
			\begin{small}
				\begin{tabular}{lcccccc}
					\toprule
					 & $\gamma=0$ & $\gamma=1$ & $\gamma=2$  & $\gamma=3$ & $\gamma=4$ & $\gamma=5$ \\
					\hline
					\hline
					Linear (no IPTW) &  0.99\% & 1.08\% & 1.36\% & 1.68\% & 2.11\% & 2.77\% \\
					MSM &  0.99\%  & 1.08\% & 1.34\% & 1.63\% & 2.02\%& 2.61\%  \\
					\hline
					RNN &  0.70\%  & 0.70\% & 0.84\% & 1.05\% & 1.24\% & 1.69\% \\
					CRN ($\lambda$ = 0)  &  0.66\%  & 0.77\% & 0.92\% & 0.95\% & 1.24\% & 1.54\% \\
					\hline
					RMSN &  0.60\%  & 0.61\% & 0.72\% & 0.81\% & 0.94\% & 1.23\% \\
					CRN &   \textbf{0.56\%}  & \textbf{0.57\%} & \textbf{0.62\%} &\textbf{ 0.67\%} & \textbf{0.87\%} & \textbf{1.20\%} \\
					\midrule 
					\midrule 
					\midrule
					 & $\gamma=6$ & $\gamma=7$ & $\gamma=8$  & $\gamma=9$ & $\gamma=10$ & \\
					\hline
					\hline
					Linear (no IPTW) &  3.55\%  & 4.15\% & 4.80\% & 5.09\% & 5.22\% &\\
					MSM &   3.30\%  & 3.79\% & 4.30\% & 4.47\% & 4.47\% &\\
					\hline
					RNN &  2.03\%  & 2.52\% & 2.88\% & 3.79\% & 4.01\% &\\
					CRN ($\lambda$ = 0)  &  1.98\%  & 2.42\% & 2.73\% & 3.17\% & 3.57\% &\\
					\hline
					RMSN &  1.70\%  & 2.18\% & 2.37\% & 2.77\% & 2.83\% &\\
					CRN &   \textbf{1.48\%}  & \textbf{1.56\%} & \textbf{2.05\%} & \textbf{2.36\%} & \textbf{2.41\%} &\\
					\bottomrule
				\end{tabular}
			\end{small}
		\end{center}
		\vskip -0.1in
	\end{table*}

		\begin{table*}[ht]
		\caption{Normalized RMSE for $\tau$-step-ahead prediction of counterfactuals. The parameter $\gamma$ measures the degree of time-dependent confounding applied.  }
		\label{tab:results_sequence_prediction}
		\vskip 0.15in
		\begin{center}
			\begin{small}
				\begin{tabular}{clccccc}
					\toprule
					& & $\gamma=1$ & $\gamma=2$ & $\gamma=3$  & $\gamma=4$ & $\gamma=5$ \\
					\hline
					\hline
					$\tau=2$& RMSN &  0.90\%  & 1.15\% & 1.53\% & 2.14\% & 2.91\% \\
					& CRN &   \textbf{0.84\%}  & \textbf{0.96\%} &\textbf{ 1.21\%} & \textbf{1.46\%} &\textbf{ 2.45\%}  \\
					\midrule
					$\tau=3$& RMSN &  0.97\%  & 1.36\% & 1.87\% & 2.44\% & 3.47\% \\
					& CRN &   \textbf{0.86\%}  & \textbf{0.96\%} &\textbf{ 1.47\%} & \textbf{1.51\%} &\textbf{ 2.84\%}  \\
					\midrule
					$\tau=4$& RMSN &  1.24\%  & 1.79\% & 2.60\% & 3.33\% & 3.88\% \\
					& CRN &   \textbf{0.91\%}  & \textbf{1.08\%} &\textbf{ 1.74\%} & \textbf{1.76\%} &\textbf{ 2.82\%}  \\
					\midrule
					$\tau=5$& RMSN &  1.51\%  & 2.13\% & 3.06\% & 4.07\% & 4.58\% \\
					& CRN &   \textbf{0.85\%}  & \textbf{1.10\%} &\textbf{ 1.73\%} & \textbf{2.00\%} &\textbf{ 3.43\%}  \\
					\midrule
					$\tau=6$& RMSN &  2.10\%  & 2.89\% & 3.06\% & 4.16\% & 6.32\% \\
					& CRN &   \textbf{1.16\%}  & \textbf{1.52\%} &\textbf{2.29\%} & \textbf{2.66\%} &\textbf{ 4.91\%}  \\
					\bottomrule
				\end{tabular}
			\end{small}
		\end{center}
		\vskip -0.1in
	\end{table*}
	
	\newpage
	
	\section{Test set generation for evaluating timing of treatment} \label{apx:test_set_timing}
	
	In order to evaluate how well the models select the correct treatment and timing of treatment we simulate counterfactual outcomes as follows. We generate 1000 test samples using the model of tumour growth described in Section 6. Let $\hist{H}_t$ be the current history of the patient and let $\tau$ be a future time horizon. For each timestep in the future, we have 4 treatment options at: no treatment ($A_0$), chemotherapy ($A_1$), radiotherapy ($A_2$), chemotherapy and radiotherapy. ($A_3$).  
	
	Using the model of tumour growth where the outcome $\vc{Y}_{t+\tau}$ is given by the volume of the tumour, we generate the following $2\tau$ counterfactuals:
	\begin{eqnarray}
			\text{Chemotherapy application} \nonumber \\
	\vc{Y}_{t+\tau} &\mid& \vc{a}_{t} = A_1, \vc{a}_{t+1} = A_0, \dots \vc{a}_{t+ \tau - 1} = A_0, \hist{H}_t \\
	\vc{Y}_{t+\tau} &\mid& \vc{a}_{t} = A_0, \vc{a}_{t+1} = A_1, \dots \vc{a}_{t+ \tau - 1} = A_0, \hist{H}_t \\
	\dots \nonumber \\
	\vc{Y}_{t+\tau} &\mid& \vc{a}_{t} = A_0, \vc{a}_{t+1} = A_0, \dots \vc{a}_{t+ \tau - 1} = A_1, \hist{H}_t\\
		\text{Radiotherapy application} \nonumber \\
	\vc{Y}_{t+\tau} &\mid& \vc{a}_{t} = A_2, \vc{a}_{t+1} = A_0, \dots \vc{a}_{t+ \tau - 1} = A_0, \hist{H}_t \\
	\vc{Y}_{t+\tau} &\mid& \vc{a}_{t} = A_0, \vc{a}_{t+1} = A_2, \dots \vc{a}_{t+ \tau - 1} = A_0, \hist{H}_t \\
	\dots \nonumber \\
	\vc{Y}_{t+\tau} &\mid& \vc{a}_{t} = A_0, \vc{a}_{t+1} = A_0, \dots \vc{a}_{t+ \tau - 1} = A_2, \hist{H}_t
	\end{eqnarray}
	We perform this for each patient in the test set and at each time $t$ in the history. For instance, for a patient with 50 timesteps in the model of tumour growth and for time horizon $\tau=3$, we generate $2 \cdot 3 \cdot 50 = 300$ counterfactuals. 

	Using the true generated couterfactual data, we select the treatment that has the lowest $\vc{Y}_{t+\tau}$ among the $\tau$ options generated for each treatment. Then, we select the time of applying treatment (among $t, t+1, \dots t+\tau-1$) that resulted in the lowest $\vc{Y}_{t+\tau}$. For each model, we generate the counterfactuals under the same treatment plans and patient histories. Then, we perform the selection of treatment and timing of treatment in the same way and we compare these with the true ones. Note that in order to account for numerical instability (two outcomes $Y_{t+\tau}$ having very similar values), we consider two outcomes the same if they are within $\epsilon = 0.001$ of each other.

\newpage
\section{Results on factual prediction on MIMIC III} \label{apx:mimic}

	Using the Medical Information Mart for Intensive Care (MIMIC III) \citep{johnson2016mimic} database consisting of electronic health records from patients in the ICU, we also show how the CRN can be used on a real medical dataset. From MIMIC III we extracted the patients on antibiotics, with trajectories up to 30 timesteps, thus obtaining a dataset with $3487$ patients. For each patient, we extracted $25$ patient covariates including lab tests and vital signs measured over time, as well as static patient features such as age and gender.

    We used a binary treatment at each timestep indicating whether the patient was administered antibiotics or not. Note that for the longitudinal covariates we used aggregate value for each day since the ICU admission. The reason for this is because antibiotic treatment is decided daily for the patient. We split the dataset into $2826/313/348$ patients for training, validation and testing respectively. We performed hyperparameter optimization on the validation patient set, using the search ranges in Table \ref{tab:hyperparameters_crn_encoder} and we again selected hyperparameters based on the error on the factual outcomes. 
    
    We estimate the individualized effect of antibiotics assigned over time on the patient's white blood cell count. A high white blood cell count is associated with severe illness and poor outcome for ICU patients \citep{waheed2003white}. Antibiotic administration in the ICU aims to reduce the white blood cell count. However, the effectiveness of the antibiotics treatment in reducing the white blood cell count is highly dependent on the time they are administered with respect to the history of the patient covariates. In this context we again have time-dependent confounders: the patient features change over time and are affected by the previous administration of antibiotics. Moreover, the history of the patient features also determines antibiotics administration and affects future patient outcomes \citep{de2018complete, ali2019rational}.
    
    In Table \ref{tab:results_sequence_prediction_mimic} we report the root mean squared error for factual prediction of the patients' white blood cell count for multiple prediction horizons $\tau$. Note that for this dataset we do not have access to counterfactual data, which is why we report error on factual predictions.

    \begin{table*}[ht]
    		\caption{RMSE for $\tau$-step-ahead prediction of factual outcomes on MIMIC III.}
    		\label{tab:results_sequence_prediction_mimic}
    		\vskip 0.05in
    		\begin{center}
    				\begin{tabular}{lcccc}
    					\toprule
    					& $\tau=1$ & $\tau=2$ & $\tau=3$  & $\tau=4$   \\
    					\hline
    					RMSN & $2.84$ & $3.87$ & $4.46$ & $4.79$ \\
    					CRN &  $2.68$ & $3.54$  &  $4.07$  & $4.67$ \\
    					\bottomrule
    				\end{tabular}
    		
    		\end{center}
    		\vskip -0.1in
    \end{table*}

    We notice that CRN also achieves better performance than RMSN in estimating factual outcomes in a real-world dataset containing electronic health records. In this context, where couterfactual data is not available, domain expert knowledge is required to validate the model's counterfactual predictions under other antibiotic treatment alternatives. This further medical validation is outside the scope of this paper.  
   
\end{document}